\documentclass[journal]{IEEEtran}

\usepackage{url}
\usepackage{times}
\usepackage{soul}
\usepackage[utf8]{inputenc}
\usepackage{graphicx}
\usepackage{amsthm}
\usepackage{amssymb}
\usepackage{amsmath}
\usepackage{amsmath,bm}
\usepackage{booktabs}
\usepackage{cite}
\usepackage{float}
\usepackage{threeparttable}

\usepackage{marginnote}
\usepackage{color}
\usepackage{mathrsfs}

\usepackage[english]{babel}
\usepackage{algcompatible}
\usepackage{hyperref}
\usepackage{algorithm}
\usepackage[noend]{algpseudocode}
\usepackage{multirow}
\usepackage{stmaryrd}
\usepackage{subfigure} 
\usepackage{epstopdf}
\usepackage[dvipsnames, svgnames, x11names]{xcolor}

\usepackage{bbm} 
\newtheorem{theorem}{Theorem}
\newtheorem{definition}{Definition}

\begin{document}

\title{BSNet: Bi-Similarity Network for Few-shot Fine-grained Image Classification}

\author{\IEEEauthorblockN{Xiaoxu Li, Jijie Wu, Zhuo Sun, Zhanyu Ma,~\emph{Senior Member, IEEE,} Jie Cao, Jing-Hao Xue,~\emph{Member, IEEE}} \thanks{
X. Li, J. Wu and J. Cao are with the School of Computer and Communication, Lanzhou University of Technology, Lanzhou 730050, China.

Z. Ma is with the Pattern Recognition and Intelligent System Laboratory, School of Artificial Intelligence, Beijing University of Posts and Telecommunications, Beijing 100876, China.

Z. Sun and J.-H. Xue are with the Department of Statistical Science, University College London, London, WC1E 6BT, U.K.

X. Li, J. Wu and Z. Sun contribute equally.
}}

\maketitle

\thispagestyle{empty}

\begin{abstract}
Few-shot learning for fine-grained image classification has gained recent attention in computer vision. Among the approaches for few-shot learning, due to the simplicity and effectiveness, metric-based  methods are favorably state-of-the-art on many tasks. Most of the metric-based  methods assume a single similarity measure and thus obtain a single feature space. However, if samples can simultaneously be well classified via two distinct similarity measures, the samples within a class can distribute more compactly in a smaller feature space, producing more discriminative feature maps. Motivated by this, we propose a so-called \textit{Bi-Similarity Network} (\textit{BSNet}) that consists of a single embedding module and a bi-similarity module of two similarity measures. After the support images and the query images pass through the convolution-based embedding module, the bi-similarity module learns feature maps according to two similarity measures of diverse characteristics. In this way, the model is enabled to learn more discriminative and less similarity-biased features from few shots of fine-grained images, such that the model generalization ability can be significantly improved. Through extensive experiments by slightly modifying established metric/similarity based networks, we show that the proposed approach produces a substantial improvement on several fine-grained image benchmark datasets. 

\textcolor{blue}{Codes are available at: https://github.com/spraise/BSNet}
\end{abstract}

{}

\section{Introduction}

Deep learning models have achieved great success in visual recognition~\cite{lecun2015deep,dvornik2019diversity,lin2019partition}. These enormous neural networks often require a large number of labeled instances to learn their parameters~\cite{simonyan2014very,gu2015recent,lin2016tube}. However, it is hard to obtain labeled image in many cases~\cite{li2020oslnet}. Furthermore, human can adapt fast based on transferable knowledge from previous experience and learn new concepts from only few observations~\cite{li2020concise}. Thus, it is of great importance to design deep learning algorithms to have such abilities~\cite{li2020remarnet}. Few-shot learning~\cite{lifchitz2019dense,sung2018learning} rises in response to these urgent demands, which aims to learn latent patterns from few labeled images~\cite{Gidaris_2019_ICCV}.

Many few-shot classification algorithms can be grouped into two branches: meta-learning based methods~\cite{bertinetto2016learning_feedforward_oneshot, finn2017maml, finn2018probabilistic_maml, qi2018low_shot_learning_imprintedweights, li2019lgm, santoro2016meta, munkhdalai2017meta} and metric-based methods~\cite{vinyals2016matching, snell2017prototypical, li2019revisiting}. 
Meta-learning based methods focus on how to learn good initializations~\cite{finn2017maml}, optimizers~\cite{chen2019a}, or parameters. 
metric-based methods focus on how to learn good feature embeddings, similarity measures~\cite{lin2017learning}, or both of them~\cite{sung2018learning}. 
Due to their simplicity and effectiveness, metric-based methods have achieved the state-of-the-art performance on fine-grained images~\cite{chen2019a,li2019revisiting}.

Fine-grained image datasets are often used to evaluate the performance of few-shot classification algorithms~\cite{chen2019a,li2019revisiting,Qiao_2019_ICCV,Tokmakov_2019_ICCV,Hao_2019_ICCV,li2019TVT}, because they generally contain many sub-categories and each sub-category includes only a small amount of data~\cite{zhang2016picking}. Due to the similarity between these small sub-categories, a key consideration in few-shot fine-grained image classification is how to learn discriminative features from few labeled images~\cite{ma2019fine,7937818,chang2020devil}.

Metric-based methods in few-shot classification are mostly built on a single similarity measure~\cite{chen2019a,li2019revisiting,Qiao_2019_ICCV,Tokmakov_2019_ICCV,Hao_2019_ICCV,li2019TVT}. Intuitively, features obtained by adapting a single similarity metric are only discriminative in a single feature space. That is, using one single similarity measure may induce certain similarity bias that lowers the generalization ability of the model, in particular when the amount of training data is small. Thus, as illustrated in Figure~\ref{fig:motivation}, if the obtained features can simultaneously adapt two similarity measures of diverse characteristics, the samples within one class can
be mapped more compactly into a smaller feature space. This will result in a model embedding two diverse similarity measures, generating more discriminative features than using a single measure. 
\begin{figure*}[!t]
\begin{center}
\includegraphics[width=6.8in]{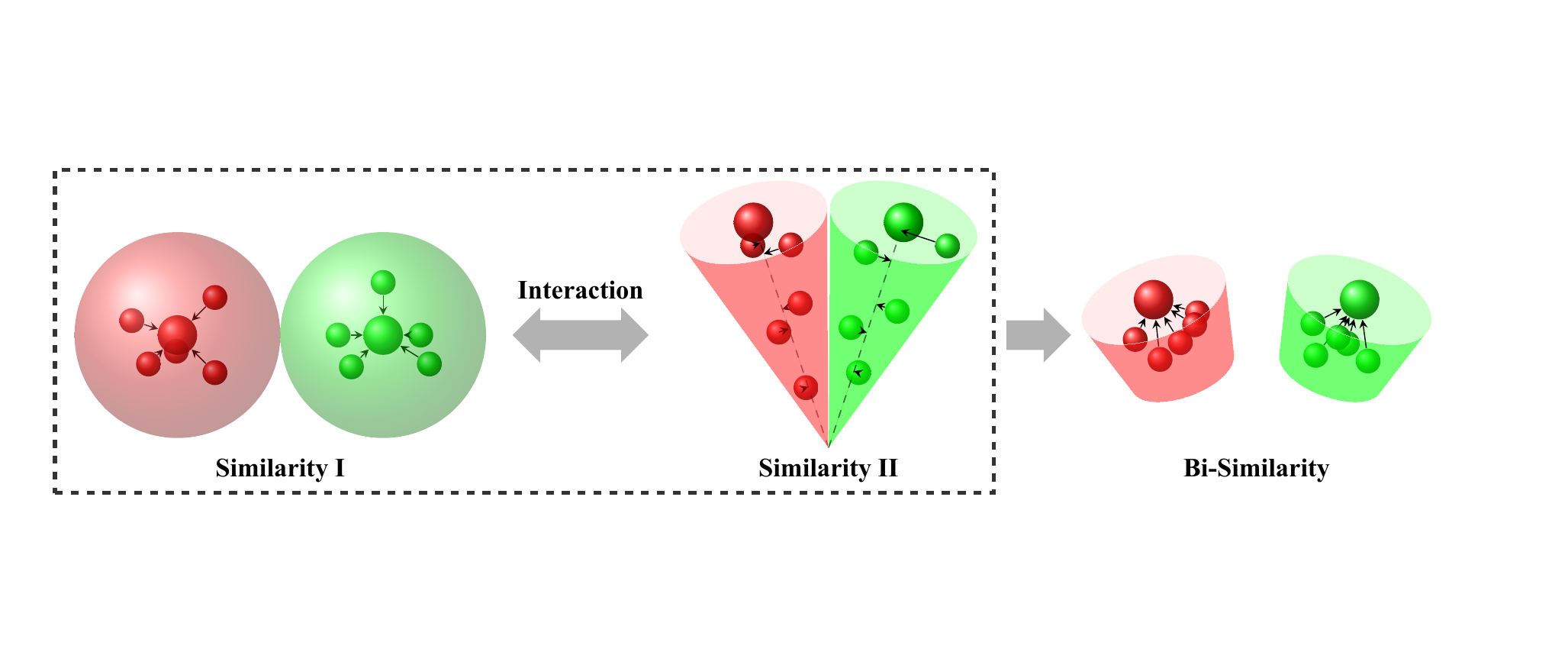}
\end{center}

\caption{Motivation of the proposed~\emph{Bi-Similarity Network} (\textit{BSNet}). Here we use the Euclidean distance and the cosine distance as the similarity measures in feature spaces. The Euclidean distance and the cosine distance are for Similarity I and Similarity II, respectively. Different colors indicate different predicted labels; larger balls are class prototypes.}
\label{fig:motivation}
\end{figure*}

Motivated by this, we develop a \textit{Bi-Similarity Network} (\textit{BSNet}). Within each task generated from a meta-training dataset, the proposed \textit{BSNet} employs two similarity measures,~\emph{e.g.}, the Euclidean distance and the cosine distance. 
To be more specific, the proposed \textit{BSNet} contains two modules as illustrated in Figure~\ref{fig:bi_similarity_network_structure}. The first module is a convolution-based embedding module, generating representations of query and support images. The learned representations are then fed forward to the second module, i.e.~our bi-similarity module, producing two similarity measurements between a query image and each class. During meta-training process, the total loss is the summation of the losses of both heads, which is used for back-propagating the network's parameters jointly. We evaluate the proposed method on several fine-grained benchmark datasets. Extensive experiments demonstrate the state-of-the-art performance of the proposed \textit{BSNet}. Our contributions are three-fold:
\begin{enumerate}
\item We propose a \textit{BSNet} that leverages two similarity measures and significantly improves the performance of four state-of-the-art few-shot classification methods~\cite{vinyals2016matching, snell2017prototypical, sung2018learning, li2019revisiting} on four benchmark fine-grained image datasets.

\item We demonstrate that the model complexity of \textit{BSNet} is less than the mean value of model complexities of two single-similarity networks, even though \textit{BSNet} contains more model parameters.

\item We demonstrate that the proposed \textit{BSNet} learns discriminative areas of the input images via visualization.

\end{enumerate}

\begin{figure*}[htbp]
\begin{center}
\includegraphics[width=6.5in]{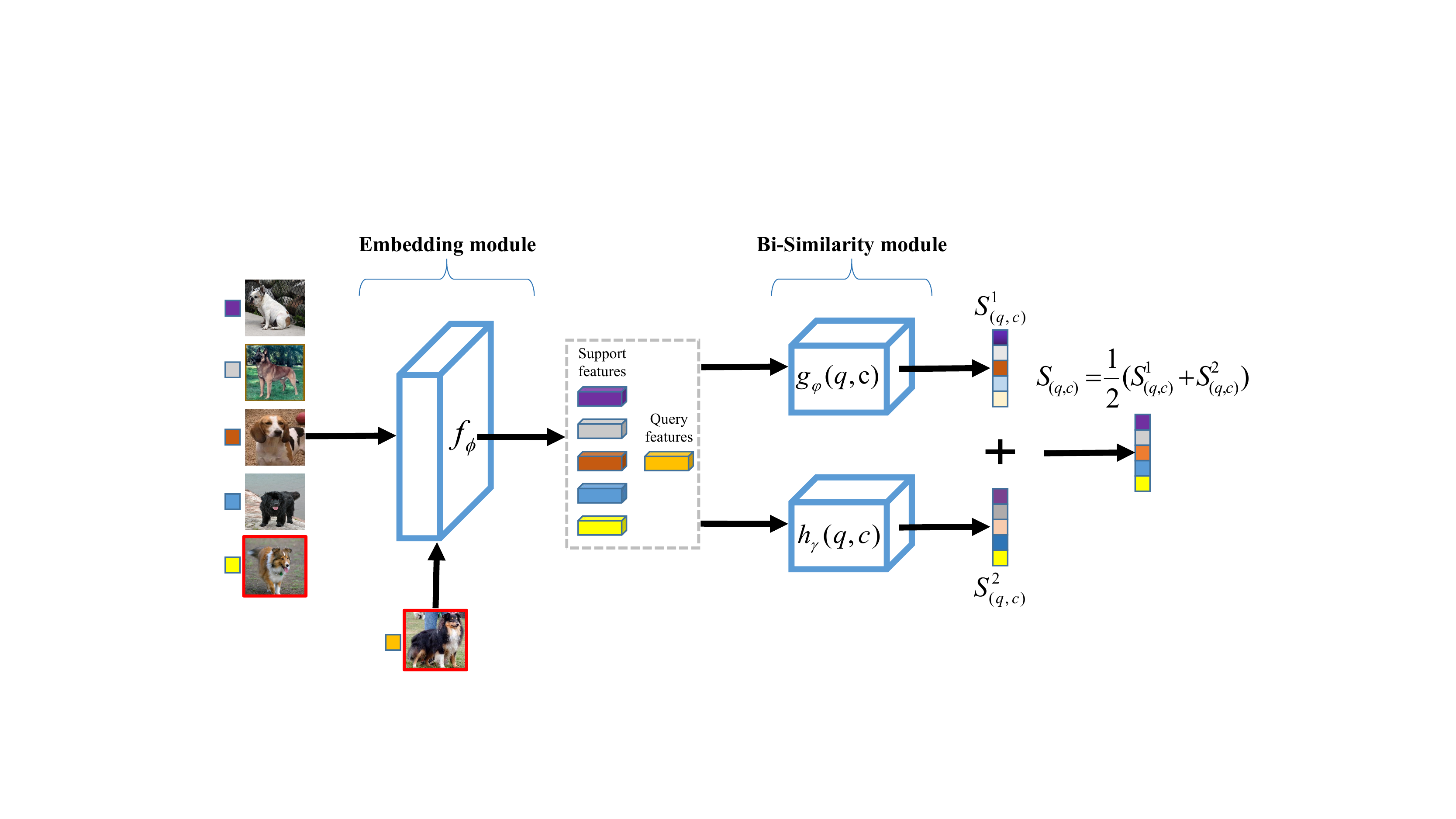}
\end{center}
\caption{Illustration of the proposed \textit{Bi-Similarity Network} (\textit{BSNet}). It consists of one embedding module $f_{\phi}$, followed by a bi-similarity module which outputs two similarity scores between a query image and $C$ class prototypes in $C$-way $K$-shot problems. $S^{k}_{(q,c)}$ denotes the $k$th similarity score between the $q$th query image $x_q$ and the $c$th class in a task.}
\label{fig:bi_similarity_network_structure}
\end{figure*}

\section{Related Work}
\label{sec:Related Works}
Recently, deep neural networks have been carefully re-designed to solve few-shot learning problems. In this work, we are particularly interested in metric-based few-shot learning models as they are most relevant to the proposed \textit{BSNet}.
Existing metric-based few-shot learning methods roughly fall into the following three groups according to their innovation: 

\paragraph{Learning feature embedding given a metric}  Koch et al.~\cite{koch2015siamese} used a Siamese convolutional neural network for one-shot leaning, in which convolution layers are of \textit{VGG}~\cite{VGG_oxford} styled structure and $L_1$ distance is used to compute the distance between image features. \textit{Matching Network}~\cite{vinyals2016matching} introduced attention mechanisms along with the cosine metric on deep features of images. Most of existing methods, for instance~\cite{snell2017prototypical, sung2018learning, allen2019infinite, kim2019edgelabelGNN}, have benefited from the episodic meta-training strategy proposed in~\cite{vinyals2016matching}, learning transferable knowledge from labeled support images to query images through this strategy. Wu et al.~\cite{Wu_2019_ICCV} added a deformed feature extractor and dual correlation attention module into the \textit{Relation Network}~\cite{sung2018learning}, and proposed position-aware relation network (or \textit{PARN}, for short). As a result, \textit{PARN} outperformed the \textit{Relation Network} on the Omniglot and Mini-ImageNet datasets.
This group of methods mainly learn a feature representation that adapts to a fixed metric. Unlike them, the proposed \textit{BSNet} learns metrics and a feature representation simultaneously.

\paragraph{Learning a prototype} Despite of similar meta-training procedures, these methods are built on different hypotheses.
\cite{snell2017prototypical} proposed the \textit{Prototype Network} that used the Euclidean distance as metric to classify query images to class-specific prototypes, producing linear decision boundaries in the learned embedding space as it implicitly assumed identical covariance matrices for all class-specific Gaussian distributions. Allen et al.~\cite{allen2019infinite} further extended the \textit{Prototype Network}~\cite{snell2017prototypical} by allowing multiple clusters for each class, and hence proposed infinite mixture prototype networks. Unlike these methods, the proposed \textit{BSNet} does not learn a prototype, but introduces the bi-similarity idea into the metric-based  few-shot learning models.

\paragraph{Learning a distance (or similarity) metric} Instead of using an existing metric, Sung et al.~\cite{sung2018learning} proposed the \textit{Relation Network} to compute similarity scores to measure the similarity between each query image and support images from each class.
Li et al.~\cite{li2019revisiting} constructed an image-to-class module after a CNN-based embedding module to compute the cosine similarity between deep local descriptors of a query image and its neighbors in each class, known as \textit{Deep Nearest Neighbor Neural Network} (or \textit{DN4} in short). Zhang et al.~\cite{zhang2019variationalfewshotlearn} learned Gaussian distributions for each class via incorporating evidence lower bound into the design of loss function and used the posterior densities of query images as the criterion, which produced quadratic decision boundaries as it allowed class-specific covariance matrices. It can be viewed as learning the Mahalanobis distance in deep feature space.
The proposed \textit{BSNet} is closely related to this group of methods. To the best of our knowledge, most current studies in few-shot learning only
focus on a single similarity measure. However, the proposed work constructs a network leveraging one module with two similarity measures.

\section{The Proposed Bi-Similarity Network}
\label{sec:Method}
\subsection{Problem Formulation}
\label{sec:problem_formulation}
Few-shot classification problems are often formalized as $C$-way $K$-shot classification problems, in which models are given $K$ labeled images from each of $C$ classes, and required to correctly classify unlabeled images. Unlike traditional classification problems where the label set in training set is identical to those in the validation and test sets, few-shot classification requires to classify novel classes after training. This requires that images used for training, validation and testing should come from disjoint label sets. To be more specific, given a dataset $\mathcal{D} = \{(x_i, y_i), y_i \in \mathcal{L} \}_{i=1}^{N}$, we divide it into three parts, namely $\mathcal{D}_{train} = \{(x_i, y_i), y_i \in \mathcal{L}_{train} \}_{i=1}^{N_{train}}$, $\mathcal{D}_{val} = \{(\tilde{x}_i, \tilde{y}_i), \tilde{y}_i \in \mathcal{L}_{val}  \}_{i = 1}^{N_{val}}$, $\mathcal{D}_{test} = \{ (x_i^{*}, y_i^{*}), y_i^{*} \in \mathcal{L}_{test} \}_{i=1}^{N_{test}}$, where $(x_i, y_i)$ represents the raw feature vector and label information for the $i$th image. The label sets, $\mathcal{L}_{train}$ and $\mathcal{L}_{val}$ and $\mathcal{L}_{test}$, are disjoint and their union is $\mathcal{L}$.

Following recent works~\cite{vinyals2016matching,snell2017prototypical,sung2018learning,chen2019a,li2019revisiting}, we organized our experiments in an episodic meta-training setting. To form a task $\mathcal{T}_{i}$ each time during episodic meta-training, we randomly select $C$ classes from $\mathcal{L}_{train}$, and randomly select $M$ images within each of these $C$ classes from $\mathcal{D}_{train}$. Within each of these selected classes, $M$ images are further separated into two sets with $K$ and $M-K$ images respectively, namely a support set $\mathcal{S}_{i}$ and a query set $\mathcal{Q}_{i}$. Similarly, we can define tasks $\tilde{\mathcal{T}}_{j}$ and $\mathcal{T}^{*}_{k}$ on $\mathcal{D}_{val}$ and $\mathcal{D}_{test}$ for meta-validation and meta-testing scenarios. We aim to train our neural networks to learn transferable deep bi-similarity knowledge from those meta-learning tasks, tune hyper-parameters through meta-validation, and report the generalization accuracy by taking the average of model's accuracy in meta-testing.

\subsection{Bi-Similarity Network}
\label{sec:bi_simlarity_networks}

\subsubsection{An Overview}

Our \textit{Bi-Similarity Network} (\textit{BSNet}) has one embedding module $f_{\phi}$, followed by a bi-similarity module consisting of two similarity measurement branches, $g_{\varphi}$ and $h_{\gamma}$, as shown in Figure~\ref{fig:bi_similarity_network_structure}.

Although we further generalize the proposed \textit{BSNet} by equipping it with other similarity measurements, i.e.\cite{vinyals2016matching, snell2017prototypical, li2019revisiting}, we only demonstrate the structure of the proposed neural network here to avoid redundant contents. The meta-training procedure of the proposed \textit{BSNet} is summarized in Algorithm~\ref{alg:bi_sim_network}. Note that in Algorithm~\ref{alg:bi_sim_network}, $\mathbbm{1}(a==b) = 1$ if and only if all elements of $a$ and $b$ are identical since we are using one-hot vectors for true/predicted labeling information for every image. A similar procedure, with the exception of updating parameters, is also used for calculating the validation and test accuracy. 

For each task $\mathcal{T}_{i}$, images $\{x_{k}^{(i)}\}_{k=1}^{M}$ from support set $\mathcal{S}_i$ and query set $\mathcal{Q}_i$ are fed into $f_{\phi}$, generating $f_{\phi}(x_{k}^{(i)})$ for all images in task $\mathcal{T}_i$. Then images' representations are further fed forward to our bi-similarity module. Let $x^{(i)}_{s,c}$ and $x^{(i)}_{q}$ denote the $s$th support image from class $c$ and the $q$th query image in task $\mathcal{T}_i$, respectively.

\subsubsection{The Training Procedure}
During the meta-training process, we assign the $q$th query image in task $\mathcal{T}_i$ according to each of the two similarity scores, $S^1_{(q,c)}$ and $S^2_{(q, c)}$. Then we obtain two one-hot vectors, $ \hat{y}^{(i, 1)}_{q}$ and $\hat{y}^{(i, 2)}_{q}$, as two predictions. Here, $S^1_{(q,c)}$ and $S^2_{(q, c)}$ refer to the two similarity scores between $x^{(i)}_{q}$ and the $c$th class, generated by the proposed \textit{BSNet}. The $\hat{y}^{(i, 1)}_{q}$ and $\hat{y}^{(i, 2)}_{q}$ only have element $1$ at $\arg \max_{c} S^{d}_{(q, c)}, d \in \{1, 2\}$, for $c \in \{ 1, \cdots, C\}$. That is, 
   \begin{equation}
   \label{eq:gen_one_hot_vec_y_hat_train}
     \begin{aligned}
          \hat{y}^{(i, 1)}_{q} &= [0, \cdots, \underbrace{1}_\text{$\arg \max_{c} S^1_{(q, c)} , c \in \{ 1, \cdots, C \}$}, \cdots, 0]  \: , \\
           \hat{y}^{(i, 2)}_{q} &= [0, \cdots, \underbrace{1}_\text{$\arg \max_{c} S^2_{(q, c)} , c \in \{ 1, \cdots, C \}$}, \cdots, 0] \: .
      \end{aligned}
   \end{equation}
Then we compute two loss values of the query image $x^{(i)}_{q}$ by comparing the two assignments $\hat{y}^{(i, 1)}_{q}$ and $\hat{y}^{(i, 2)}_{q}$ with the one-hot vector $y_{q}^{(i)}$ representing its ground-truth label, and take the average of the two values for back-propagation. The two loss values are given by
\begin{equation}
     \label{eq:bi_sim_nets_trainingLoss}
        \begin{aligned}
              & l_q^{(i, 1)} = \sum_j^{C}(\hat{y}_{q,j}^{(i, 1)} - y_{q,j}^{(i)})^2, q = 1, \cdots, |\mathcal{Q}_i| \: , \\
              & l_q^{(i, 2)} = \sum_j^{C}(\hat{y}_{q,j}^{(i, 2)} - y_{q,j}^{(i)})^2, q = 1, \cdots, |\mathcal{Q}_i|  \: , \\
               & l_q^{(i)} = ( l_q^{(i, 1)} +  l_q^{(i, 2)}) \: ,
        \end{aligned}
\end{equation}
where $|\mathcal{Q}_i|$ denotes the total number of query images in $\mathcal{Q}_i$.

\subsubsection{The Validation/Testing Procedure} During the validation/test process, we assign the query image to the class with the maximum average similarity score and generate a corresponding one-hot vector $\hat{y}^{(i)}_{q}$. That is,
\begin{equation}
\label{eq:gen_one_hot_vec_y_hat_testing}
    \begin{aligned}
       \hat{y}^{(i)}_{q} &= [0, \cdots, \underbrace{1}_\text{$\arg \max_{c} \frac{1}{2}\bigg(   S^1_{(q, c)} +   S^2_{(q, c)} \bigg), c \in \{ 1, \cdots, C \}$}, \cdots, 0] \: .
    \end{aligned}
\end{equation}

\begin{algorithm}[H]
\caption{Meta-Training procedure for BSNet under $C$-way $K$-shot scenarios}\label{alg:bi_sim_network} 
\begin{algorithmic}[1]
\Require{Input} Model $\mathcal{M}$, Optimizer, Number of episodes $B$, $\mathcal{D}_{train}$, $C$, $K$.
\For{$i$ from 1 to $B$}
\State Generate a task $\mathcal{T}_i$ from $\mathcal{D}_{train}$ \Comment{Details in Sec.~\ref{sec:problem_formulation}}
\State Further split $\mathcal{T}_i$ into $\mathcal{S}_{i}$ and $\mathcal{Q}_i$ \Comment{$C \times K$ images for $\mathcal{S}_i$, $C \times (M-K)$ images for $\mathcal{Q}_i$}
\State $L^{(i)} \gets 0$ \Comment{Initializing loss for the $i$th task}
\State $A^{(i)} \gets 0$ \Comment{Initializing accuracy for the $i$th task}
\For{$x_q^{(i)}$ in $\mathcal{Q}_i$} \Comment{$x_q^{(i)}$ is the $q$th query sample in $\mathcal{Q}_i$}
   \State Compute the $C \times 2$ similarity scores for $x_q^{(i)}$, i.e. $S^{d}_{(q,c)}$, for $d=1, 2$ and $c = 1, \cdots, C$.
   \State Compute two class assignments for $x_q^{(i)}$: $\arg \max_c S^1_{(q,c)} $ and $\arg \max_c S^2_{(q,c)}$
   \State Generate two corresponding one-hot vectors $\hat{y}_q^{(i, 1)}$ and $\hat{y}_q^{(i, 2)}$  according to Equation (\ref{eq:gen_one_hot_vec_y_hat_train}).
   \State Compute the loss of $\mathcal{M}$ for $\mathcal{S}_i$ and $x_q^{(i)}$ according to Equation (\ref{eq:bi_sim_nets_trainingLoss}), denoted by $l_q^{(i)}$. 
   \State $L^{(i)} \gets L^{(i)} + l_q^{(i)}$
   \State $A^{(i)} \gets A^{(i)} + \mathbbm{1}( \hat{y}_q^{(i)} == y_q^{(i)})$ \Comment{$\hat{y}_{q}^{(i)}$ is $\mathcal{M}$'s prediction for $x_q^{(i)}$ }
\EndFor

\State $L^{(i)} \gets \frac{L^{(i)}}{C \times |\mathcal{Q}_i|}$
\State $A^{(i)} \gets \frac{A^{(i)}}{|\mathcal{Q}_i|}$

\State Use $L^{(i)}$ to update $\mathcal{M}$'s parameters according to the optimizer \Comment{e.g Adam optimizer}
\EndFor
\State Stop training $\mathcal{M}$. Output mean accuracy $\frac{1}{B}\sum_{i=1}^{B} A^{(i)}$.
\end{algorithmic}
\end{algorithm}

\subsection{The Design of BSNet}
\label{sec:details_of_architecture}
In the proposed \textit{BSNet}, the feature embedding module and the bi-similarity module are key parts that affect overall performance. In terms of the feature embedding module $f_{\phi}$, we adopt convolution-based architecture, and will introduce the detailed design in Sec.~\ref{sec:Implementation Details}.

For the design of bi-similarity module, both of the two chosen similarity measurements should have good fitting ability. We take the \textit{relation module} as an example. The \textit{relation module} was proposed in~\cite{sung2018learning}, in which a query image's representation $f_{\phi}(x^{(i)}_{q})$ is concatenated to every support image's representation  $f_{\phi}(x^{(i)}_{s,c})$. In terms of $C$-way $1$-shot scenario ($|\mathcal{S}_i| = 1$), for each query image $x_{q}^{(i)}$, $C \times 1$ concatenated features are directly fed forward into two further convolution blocks and two fully connected layers to get $C$ similarity scores: 
\begin{equation}
\label{eq:deep_relation_score_K1}
    \begin{aligned}
       S^1_{(q, c)} = g_{\varphi} \bigg(  [f_{\phi}(x^{(i)}_{1,c}) || f_{\phi}(x^{(i)}_{q} ] \bigg), \; c = 1, \cdots C,
    \end{aligned}
\end{equation}
where $||$ denotes the concatenation operator.
While for the $K$-shot scenario, we take the mean of support images' feature maps from each class to act as $C$ prototypes and concatenate them ahead of the query image's feature map, and still obtain $C$ similarity scores for this query image:
\begin{equation}
\label{eq:deep_relation_score_K>1}{}{}
    \begin{aligned}
        S^1_{(q, c)} = g_{\varphi} \bigg(  [ \frac{1}{K} \sum_{s=1}^{K}f_{\phi}(x^{(i)}_{s, c}) || f_{\phi}(x^{(i)}_{q} ] \bigg), \; c = 1, \cdots C.
    \end{aligned}
\end{equation}

In addition, we also introduce a cosine similarity module designed by our own (details in Sec.~\ref{sec:Implementation Details}). Actually, many previous studies in few-shot classification have adopted the cosine distance~\cite{chen2019a, li2019revisiting}.
In this work, we combine the self-designed cosine similarity module with the similarity modules of \textit{Matching Network}~\cite{vinyals2016matching}, \textit{Prototype Network}~\cite{snell2017prototypical}, \textit{DN4}~\cite{li2019revisiting} in order to justify effectiveness of the proposed \textit{BSNet}. The details of these similarity modules are present in Sec.~\ref{sec:Implementation Details}.

Our cosine similarity module contains two convolution blocks ($h^{em}_{\gamma}$) and a cosine similarity layer ($h^{cos}_{\gamma}$). The cosine similarity between a query image $x_{q}^{(i)}$ and the $c$th class, conditioned on the support images from class $c$, is given as
\begin{equation}
\label{eq:cosine_similarity}
    \begin{aligned} 
         S^2_{(q, c)} =&  h^{cos}_{\gamma}\bigg(  h^{em}_{\gamma}\big( \frac{1}{K} \sum_{s=1}^{K} f_{\phi}(x^{(i)}_{s,c})\big), h^{em}_{\gamma}\big( f_{\phi}(x^{(i)}_{q})\big) \bigg) \: .
    \end{aligned}
\end{equation}
Note that a convolution block in our experiments refers to a convolution layer, a batch norm layer, a \textit{ReLU} layer and with or without a pooling layer. 
After the model computes two similarity scores between the query image and each class, we have two strategies for training and validation/testing processes, respectively, which have already been demonstrated in Sec.~\ref{sec:bi_simlarity_networks}. 


\subsection{The Empirical Rademacher Complexity of BSNet}

We start by introducing some notation. We define a single similarity network as a network consisting of a feature embedding module and a similarity module. Let $I \left ( x; \mathbf{W}_{em1}, \mathbf{W}_{s1} \right ) \in \mathcal{I}$, and $J \left ( x; \mathbf{W}_{em2}, \mathbf{W}_{s2} \right) \in \mathcal{J}$ denote two different similarity networks which contain the feature embedding module with the same structure, where $\mathbf{W}_{em1}$, $\mathbf{W}_{em2}$ are the parameters of the feature embeddings of $I$ and $J$, respectively, and $\mathbf{W}_{s1}$ and $\mathbf{W}_{s2}$ are the parameters of the similarity modules of $I$ and $J$, respectively. We define the \textit{BSNet} as learning a common feature embedding for $I$ and $J$ and adopting their similarity modules. Therefore, the Bi-similarity can be denoted as  $Z\left ( x; \mathbf{W}_{s1}, \mathbf{W}_{s2}, \mathbf{W}_{em} \right ) \in \mathcal{Z}$, where $\mathbf{W}_{em}$ is the parameter of the common feature embedding. 
The Rademacher complexity measures the richness of a family of functions which is defined as follows:
\begin{definition}~\cite{mohri2018foundations} Let $\mathcal{I}$ denote a family of real valued function and $S=\{x_1,x_2,\ldots,x_M\}$ be a fixed sample of size $M$. Then, the empirical Rademacher complexity of $\mathcal{I}$ with respect to $S$ is defined as
$\hat{R}_{S}\left(\mathcal{I}\right) = \mathtt{\bf{E}}_{\mathbf{\sigma}}\big[\sup_{I\in\mathcal{I}} \frac{1}{M}\sum_{i=1}^{M}\sigma_{i}I(x_i)\big]$, 
where $\mathbf{\sigma} = \{\sigma_{1},\sigma_{2},\ldots,\sigma_{M}\}$ with $\sigma_{i}$ independently distributed according to $P(\sigma_{i}=1)= P(\sigma_{i}=-1) = \frac{1}{2}$.
\end{definition}

\begin{theorem}
\label{theorem}

Let $ \hat{R}_{S}\left(\mathcal{I}\right)$, $ \hat{R}_{S}\left(\mathcal{J}\right)$, and $\hat{R}_{S}\left(\mathcal{Z}\right)$ denote the Rademacher complexities of the single similarity network family $\mathcal{I}$, the single similarity network family $\mathcal{J}$, and the Bi-similarity 
network family $\mathcal{Z}$.  We have $ \hat{R}_{S}\left(\mathcal{Z}\right) \leq  \frac{1}{2} \left( \hat{R}_{S}\left(\mathcal{I}\right) + \hat{R}_{S}\left(\mathcal{J}\right)\right)$.

\end{theorem}

\begin{proof}
 Here, we introduce a network family $\mathcal{P}$, $P(x)= \frac{1}{2}\left( I(x) + J(x)\right)\in \mathcal{P}$, then $P(x)$ can be denoted by $P \left( x; \mathbf{W}_{s1}, \mathbf{W}_{s2}, \mathbf{W}_{em1}, \mathbf{W}_{em2}\right)$.
Following the notation above, for $\forall Z \in \mathcal{Z}$, $\exists P \in \mathcal{P}$ subject to  $Z(x)= P(x)$, so we have $\mathcal{Z} \subset \mathcal{P}$. Thus, we have
\begin{equation}
     \label{eq:proof}
        \begin{aligned}
             &\hat{R}_{S}\left(\mathcal{Z}\right)  = \mathtt{\bf{E}}_{\mathbf{\sigma}}\left[\sup_{Z\in\mathcal{Z}} \frac{1}{M}\sum_{i=1}^{M}\sigma_{i}Z(x_i)\right]  \\
             &  \leq \mathtt{\bf{E}}_{\mathbf{\sigma}}\left[\sup_{P\in\mathcal{P}} \frac{1}{M}\sum_{i=1}^{M}\sigma_{i} P\left(x_i\right)\right]\\
              &  = \mathtt{\bf{E}}_{\mathbf{\sigma}}\left[\sup_{I\in\mathcal{I},J\in\mathcal{J}} \frac{1}{M}\sum_{i=1}^{M}\sigma_{i} \left(\frac{1}{2}I\left(x_i\right)+\frac{1}{2}J\left(x_i\right)\right)\right]  \\
              &  = \frac{1}{2}\mathtt{\bf{E}}_{\mathbf{\sigma}}\left[\sup_{I\in\mathcal{I}} \frac{1}{M}\sum_{i=1}^{M}\sigma_{i}I\left(x_i\right) \right] + \frac{1}{2}\mathtt{\bf{E}}_{\mathbf{\sigma}}\left[\sup_{J\in\mathcal{J}} \frac{1}{M}\sum_{i=1}^{M}\sigma_{i}J\left(x_i\right) \right]\\
              & = \frac{1}{2} \left( \hat{R}_{S}\left(\mathcal{I}\right) + \hat{R}_{S}\left(\mathcal{J}\right)\right).           
        \end{aligned}      
\end{equation}\end{proof}

This theorem states that even though \textit{BSNet} has more model parameters, the Rademacher complexity of the proposed \textit{BSNet} is no more than the average of the Rademacher complexities of two individual networks. 


\begin{table*}[h]
\centering
\caption{Experiment Configuration (shown in columns). }
\label{tab:experiment_detials}
\resizebox{\textwidth}{!}{%
\begin{tabular}{@{}c|c|c|c|c|c@{}}
\toprule[1pt]
\multicolumn{6}{c}{Experiment Configuration}                                                                       \\ \hline
                   & Matching Network~\cite{vinyals2016matching}     & Prototype Network~\cite{snell2017prototypical}         & Relation Network~\cite{sung2018learning}   & Cosine Network     & DN4~\cite{li2019revisiting}                 \\ \hline
Input size         & \multicolumn{5}{c}{$3\times84\times84$}  \\ \hline
Embedding module $f_{\phi}$  & \multicolumn{4}{c|}{Conv4}                                         & Con64F                 \\ \hline
Feature map size      & \multicolumn{4}{c|}{$64\times19\times19$}                            & $64\times 21\times 21$ \\ \hline
Similarity module  & Matching module & Prototype module & Relation module & Cosine module & Image-to-Class module     \\ \hline
Optimizer          & \multicolumn{5}{c}{Adam (initial learning rate = $10^{-3}$, weight decay = 0)}                                       \\ \hline
Loss function     & NLL Loss         & Cross Entropy Loss & \multicolumn{2}{c|}{MSE Loss}    & Cross Entropy Loss       \\ \hline
Data augment     & \multicolumn{5}{c}{Random Sized Crop, Image Jitter, Random Horizontal Flip}                        \\ \hline
Training episodes     & \multicolumn{4}{c|}{1-shot: $60,000$, 5-shot: $40,000$}                & $300, 000$             \\ \hline
Training query sample size & \multicolumn{4}{c|}{16}                                              & 1-shot: 15, 5-shot: 10   \\ \hline
\end{tabular}%
}
\end{table*}

\begin{table*}[htp]
\renewcommand
\arraystretch{1.3}
\centering
\caption{Five-way few-shot classification performance on the \textit{Stanford-Cars} (Cars), \textit{Stanford-Dogs} (Dogs), \textit{FGVC-Aircraft} (Aircraft) and \textit{CUB-200-2011} (CUB) datasets. The methods include: \textit{Matching Network} (Matching), \textit{Prototype Network} (Prototype), \textit{Relation Network} (Relation) and \textit{Deep Nearest Neighbor Neural Network} (DN4) and the proposed \textit{BSNet}. We report the mean accuracy for each method, along with its $95\%$ confidence interval.}
\label{tab:5way_all4metrics}
\begin{tabular}{cllllllll}
\toprule[1pt]
                                             
\multirow{2}{*}{\it{Model}}            & \multicolumn{4}{c}{5-Way 5-shot Accuracy (\%)}                                                                                       & \multicolumn{4}{c}{5-Way 1-shot Accuracy (\%)}                        \\  \cmidrule(r){2-5}  \cmidrule(r){6-9} 
                                       & \multicolumn{1}{c}{\it{Cars}}& \multicolumn{1}{c}{\it{Dogs}}  & \multicolumn{1}{c}{\it{Aircraft}}& \multicolumn{1}{c}{\it{CUB}} & \multicolumn{1}{c}{\it{Cars}}  & \multicolumn{1}{c}{\it{Dogs}}& \multicolumn{1}{c}{\it{Aircraft}}& \multicolumn{1}{c}{\it{CUB}}\\ \toprule

Matching~\cite{vinyals2016matching}     & \textbf{64.74 $\pm$ 0.72} & 59.79 $\pm$ 0.72          & 73.75 $\pm$ 0.69          & 74.57 $\pm$ 0.73          & 44.73 $\pm$ 0.77          & \textbf{46.10 $\pm$ 0.86}  & \textbf{56.74 $\pm$ 0.87} & 60.06 $\pm$ 0.88          \\
BSNet (M\&C)                             & 63.58 $\pm$ 0.75          & \textbf{61.61 $\pm$ 0.69} & \textbf{75.92 $\pm$ 0.76} & \textbf{74.68 $\pm$ 0.71} & \textbf{44.93 $\pm$ 0.80} & 45.91 $\pm$ 0.81           & 56.53 $\pm$ 0.81          & \textbf{60.73 $\pm$ 0.94} \\ \bottomrule
Prototype~\cite{snell2017prototypical}  & 62.14 $\pm$ 0.76          & 61.58 $\pm$ 0.71          & 71.27 $\pm$ 0.67          & 75.06 $\pm$ 0.67          & 36.54 $\pm$ 0.74          & 40.81 $\pm$ 0.83           & 46.68 $\pm$ 0.81          & 50.67 $\pm$ 0.88          \\ 
BSNet (P\&C)                             & \textbf{63.72 $\pm$ 0.78} & \textbf{62.61 $\pm$ 0.73} & \textbf{77.35 $\pm$ 0.68} & \textbf{76.34 $\pm$ 0.65} & \textbf{44.56 $\pm$ 0.83} & \textbf{43.13 $\pm$ 0.85}  & \textbf{52.48 $\pm$ 0.88} & \textbf{55.81 $\pm$ 0.97} \\ \bottomrule
Relation~\cite{sung2018learning}        & 68.52 $\pm$ 0.78          & 66.20 $\pm$ 0.74          & 75.18 $\pm$ 0.74          & 77.87 $\pm$ 0.64          & 46.04 $\pm$ 0.91          & 47.35 $\pm$ 0.88           & 62.04 $\pm$ 0.92          & 63.94 $\pm$ 0.92          \\ 
BSNet (R\&C)                             & \textbf{73.47 $\pm$ 0.75} & \textbf{68.60 $\pm$ 0.73} & \textbf{80.25 $\pm$ 0.67} & \textbf{80.99 $\pm$ 0.63} & \textbf{54.12 $\pm$ 0.96} & \textbf{51.06 $\pm$ 0.94}  & \textbf{64.83 $\pm$ 1.00} & \textbf{65.89 $\pm$ 1.00} \\ \bottomrule
DN4~\cite{li2019revisiting}             & \textbf{87.47 $\pm$ 0.47} & 69.81 $\pm$ 0.69          & \textbf{84.07 $\pm$ 0.65} & 84.41 $\pm$ 0.58          & 34.12 $\pm$ 0.68          & 39.08 $\pm$ 0.76           & 60.65 $\pm$ 0.95          & 57.45 $\pm$ 0.89          \\ 
BSNet (D\&C)                             & 86.88 $\pm$ 0.50          & \textbf{71.90 $\pm$ 0.68} & 83.12 $\pm$ 0.68          & \textbf{85.39 $\pm$ 0.56} & \textbf{40.89 $\pm$ 0.77} & \textbf{43.42 $\pm$ 0.86}  & \textbf{62.86 $\pm$ 0.96} & \textbf{62.84 $\pm$ 0.95} \\ \bottomrule
\end{tabular}
\end{table*}

\section{Experimental Results and Discussions}
\label{sec:Experimental_Results}
We evaluate the proposed approach on few-shot classification tasks. This evaluation
serves four purposes: 
\begin{itemize}
  \item To compare the proposed \textit{BSNet} with the state-of-the-art methods on few-shot fine-grained classification (Sec.~\ref{sec:sota1});
   \item To investigate the generalization ability of \textit{BSNet} by changing the backbone network (Sec.~\ref{sec:Backbones});
    \item To study the effectiveness of each similarity module in the proposed \textit{BSNet} (Sec.~\ref{sec:on Effectiveness of of Two-branch});
    \item To demonstrate that the proposed \textit{BSNet} can learn a reduced number of class-discriminative features (Sec.~\ref{sec:Feature Visualization});
     \item To investigate the effect of changing similarity modules on \textit{BSNet} (Sec.~\ref{sec:ChangingSimilarityModules}).
\end{itemize}

\subsection{Datasets}
\label{sec:datasets}
We conducted all the experiments on four benchmark fine-grained datasets, \textit{FGVC-Aircraft}, \textit{Stanford-Cars}, \textit{Stanford-Dogs} and \textit{CUB-200-2011}. For each dataset, we divided them into meta-training set $\mathcal{D}_{train}$, meta-validation set $\mathcal{D}_{val}$ and meta-test set $\mathcal{D}_{test}$ in a ratio of $2:1:1$. All images in the four datasets are resized to $84 \times 84$. We did not use boundary boxes for all the images.

\textit{\textbf{FGVC-Aircraft}}~\cite{maji2013fine} is a classic dataset in fine-grained image classification and recognition research. It contains four-level hierarchical notations: \textit{model}, \textit{variant}, \textit{family} and \textit{manufacturer}. According to the \textit{variant} level, it can be divided into 100 categories (classification annotation commonly used in fine-grained image classification), which are the annotation level we used. \textit{\textbf{Stanford-Cars}}~\cite{krause20133d} is also a benchmark dataset for fine-grained classification, which contains 16,185 images of 196 classes of cars, and categories of this dataset are mainly divided based on the brand, model, and year of the car. \textit{\textbf{Stanford-Dogs}}~\cite{khosla2011novel} is a challenging and large-scale dataset that aims at fine-grained image categorization, including 20,580 annotated images of 120 breeds of dogs from around the world. \textit{\textbf{CUB-200-2011}}~\cite{wah2011caltech} contains 11,788 images from 200 bird species. This dataset is also a classic benchmark for fine-grained image classification.

\subsection{Implementation Details}
\label{sec:Implementation Details}
In order to prove the effectiveness of the proposed method, we choose several metric/similarity based few-shot classification networks, \textit{Matching Network}~\cite{vinyals2016matching}, 
\textit{Prototype Network}~\cite{snell2017prototypical}, \textit{Relation Network}~\cite{sung2018learning}, \textit{Deep Nearest Neighbor Neural Network} (or \textit{DN4}, for short)~\cite{li2019revisiting}, to be the baseline of the proposed method. At the same time, we combine these diverse similarity measurement modules with a self-designed cosine module to construct the proposed \textit{BSNet}. We present details of our experiments in this section, which include information about experiment parameters, the matching module, the prototype module, the relation module, the image-to-class module and the bi-similarity module. Together with Table~\ref{tab:experiment_detials}, it provides a detailed summary for all implementations in our experiments.

All methods are trained from scratch. Some important parameters during training are shown in Table~\ref{tab:experiment_detials}. In particular, for the optimizer of \textit{DN4} and bi-similarity experiments involved \textit{DN4}, the learning rate is reduced by half every $100,000$ episodes, while the other methods do not adopt this strategy. In all experiments, we use the meta-validation set to select the optimal setting. At the meta-testing stage, we report the mean accuracy of 600 randomly generated testing episodes as well as $95\%$ confidence intervals. i.e., $ \mathrm{Mean}\pm 1.96 \times \frac{\mathrm{Std}}{\sqrt{600}} $. Particularly, in terms of \textit{DN4}, the above test procedure is repeated five times, and the mean accuracy with 95\% confidence intervals are reported. We will introduce the embedding module and several similarity measurement modules used in our experiments in detail below, as well as the self-defined Cosine module and Bi-similarity module proposed by us.

\textbf{Embedding module:}
We construct four convolution blocks to be our embedding module $f_{\phi}$ to learn a shared representation in the proposed \textit{BSNet}. Following~\cite{chen2019a}, we used the embedding module \textit{Conv4} for \textit{Matching Network}, \textit{Prototype Network}, \textit{Relation Network} and \textit{Cosine Network}. Specifically, each convolution block of \textit{Conv4} has $3 \times 3$ convolution of 64 filters, followed by batch normalization and a \textit{ReLU} activation function. 
Following~\cite{li2019revisiting}, we used the embedding module \textit{Conv64F} for DN4. In \textit{Conv64F}, each convolution block employs $64$ $3 \times 3$ convolution filters without padding, followed by batch normalization and a \textit{Leaky ReLU} activation function. The $2 \times 2$ max-pooling layer is used in the first two blocks. After passing through the \textit{Conv64F} module, a $3 \times 84 \times 84$ input image is now represented by a $64 \times 21 \times 21$ feature map.


\textbf{Matching module:}
The matching module is composed of three parts: a memory network, a metric module and an attention module. The memory network uses a simple bidirectional \textit{LSTM}, which further processes the support features from the embedding module $f_{\phi}$, so that these features can contain the context information. The metric module uses a cosine distance measure, and a softmax layer is used in the attention module. In the process of model training, a single support set and a query sample will be used as input such that feature extraction will be carried out at the same time. After that, these features will be further processed through the memory network, and then the final prediction value of a query image depends on distance measurements and the attention module.

\textbf{Prototype module:}
The prototype module is a module which includes an Euclidean distance algorithm to calculate the distance of query features and support features from the embedding module $f_{\phi}$. The specific calculation formula is $-{\left\Vert f_{\phi}(x_q) - f_{\phi}(x_s) \right\Vert}^2$, where $f_{\phi}(x_q)$ represents the flattened vector of a query feature and $f_{\phi}(x_s)$ represents the flattened vector of a support prototype feature. For the 1-shot classification, the feature of a support sample within a class is the prototype of the class. For 5-shot, the mean of the features of all 5 support samples in the same class is the corresponding class prototype. The Euclidean distance between a query feature and a class prototype represents the dissimilarity between the query sample and the class.

\textbf{Relation module:}
The relation module consists of two convolution blocks and two full connection layers (or  \textit{FC}, for short). Each convolution block is composed of $3 \times 3$ convolution of 64 filters, followed by batch normalization, a \textit{ReLU} non-linearity function and a $2 \times 2$ max-pooling. The padding parameter of both blocks is set to be 0. The first \textit{FC} layer is followed by a \textit{ReLU} function, and the second \textit{FC} layer is followed by a sigmoid function. The input sizes of the first and the second \textit{FC} layers are 576 and 8, respectively, and the final output size is 1. During the training process, we concatenated the feature of a query feature to every class prototype (i.e. mean of support image's features within each class), resulting in 128-dimensional relation pairs. By passing these relation pairs into the relation module, it computes similarity scores between the query sample and each class.

\textbf{Image-to-Class module:}
Different from other methods, we use \textit{Conv64F} to be the embedding module of image-to-class module, of which the output size is $64 \times 21 \times 21$. We divide the features into 441 ($21 \times 21$) 64-dimensional local features. The image-to-class module consists of a cosine measurement module and a \textit{K-Nearest-Neighbours} classifier (\textit{KNN} for short). The cosine measurement module calculates the cosine distance between each local features of the query sample and each local features of the support samples. In terms of 5-way 1-shot classification problems, we calculate the cosine distance of the local features between the query sample and each of the 5 support samples, get the cosine distance of size $5 \times 441 \times 441$. Then the \textit{KNN} classifier takes the category corresponding to the maximum of sum of top-3 cosine distances in each class as the prediction category of the query sample (i.e.~using $K=3$). For 5-shot experiments, we provide 5 support features of each class with 2205 ($21 \times 21 \times 5$) 64-dimensional local features.
\begin{table}[!t]
\renewcommand
\arraystretch{1.3}
\centering
\caption{Five-way few-shot classification performance of the proposed \textit{BSNet} by adjusting the Weights in the loss function, $l_q$, on the CUB-200-2011 dataset. Here $l_q =  \lambda \times l_q^{(1)} + \beta \times l_q^{(2)}$. We report the mean accuracy of 600 randomly generated testing episodes for \textit{BSNet} (R\&C), with its $95\%$ confidence interval.}
\label{tab:tuneweights}
\begin{tabular}{llll}
\toprule[1pt]
\multirow{3}{*}{$\lambda$} & \multirow{3}{*}{$\beta$}            & \multicolumn{2}{c}{5-Way Accuracy (\%)}                               \\ \cline{3-4}

           &            & \multicolumn{2}{c}{\it{CUB-200-2011} }                \\
           &             & 1-shot                    & 5-shot                    \\ \toprule
1          & 0.1                & 65.65 $\pm$ 0.97          & 78.80 $\pm$ 0.70          \\
1          & 0.3              & 63.86 $\pm$ 0.91          & 78.83 $\pm$ 0.62          \\
1          & 0.5                & 64.75 $\pm$ 0.97          & 80.01 $\pm$ 0.65          \\
1          & 0.7                  & 63.55 $\pm$ 0.97          & 79.08 $\pm$ 0.64          \\
1          & 0.9               & 64.48 $\pm$ 1.00          & 79.08 $\pm$ 0.63          \\
0.1        & 1                   & 64.36 $\pm$ 0.94          & 79.22 $\pm$ 0.67          \\
0.3        & 1                & 64.38 $\pm$ 1.00          & 79.55 $\pm$ 0.68          \\
0.5        & 1            & 64.96 $\pm$ 0.98          & 79.41 $\pm$ 0.68          \\
0.7        & 1                   & 63.97 $\pm$ 1.02          & 78.59 $\pm$ 0.67          \\
0.9        & 1                 & 64.92 $\pm$ 0.96          & 78.53 $\pm$ 0.69          \\
\textbf{1} & \textbf{1}       & \textbf{65.89 $\pm$ 1.00} & \textbf{80.99 $\pm$ 0.63} \\ \bottomrule

\end{tabular}
\end{table}

\begin{table*}[!htp]
\renewcommand
\arraystretch{1.3}
\centering
\caption{Few-shot classification of changing the Backbone network on the Stanford-Cars and CUB-200-2011 datasets. We report the mean accuracy for each method, along with its $95\%$ confidence interval.}
\begin{tabular}{cccccc}
\toprule[1pt]
\multirow{3}{*}{Backbone}                      & \multirow{3}{*}{Model}      & \multicolumn{4}{c}{5-Way Accuracy (\%)}                                                                       \\ \cline{3-6}
                                               &                             & \multicolumn{2}{c}{5-shot}                  & \multicolumn{2}{c}{1-shot}                \\
                                               &                             & \multicolumn{1}{c}{\it{Stanford-Cars}}      & \multicolumn{1}{c}{\it{CUB-200-2011}}      & \multicolumn{1}{c}{\it{Stanford-Cars}}      & \multicolumn{1}{c}{\it{CUB-200-2011}}                    \\ \toprule

\multirow{8}{*}{Conv4}      & Matching Network~\cite{vinyals2016matching}    & \textbf{64.74 $\pm$ 0.72}                   & 74.57 $\pm$ 0.73                           & 44.73 $\pm$ 0.77                   & 60.06 $\pm$ 0.88                            \\
                            & BSNet (M\&C)                                & 63.58 $\pm$ 0.75                            & \textbf{74.68 $\pm$ 0.71}                  & \textbf{44.93 $\pm$ 0.80}          & \textbf{60.73 $\pm$ 0.94}                   \\ \cline{2-6}
                            & Prototype Network~\cite{snell2017prototypical} & 62.14 $\pm$ 0.76                            & 75.06 $\pm$ 0.67                           & 36.54 $\pm$ 0.74                   & 50.67 $\pm$ 0.88                            \\
                            & BSNet (P\&C)                                & \textbf{63.72 $\pm$ 0.78}                   & \textbf{76.34 $\pm$ 0.65}                  & \textbf{44.56 $\pm$ 0.83}          & \textbf{55.81 $\pm$ 0.97}                   \\ \cline{2-6}
                            & Relation Network~\cite{sung2018learning}       & 68.52 $\pm$ 0.78                            & 77.87 $\pm$ 0.64                           & 46.04 $\pm$ 0.91                   & 63.94 $\pm$ 0.92                            \\
                            & BSNet (R\&C)                   & \textbf{73.47 $\pm$ 0.75}                   & \textbf{80.99 $\pm$ 0.63}                  & \textbf{54.12 $\pm$ 0.96}          & \textbf{65.89 $\pm$ 1.00}                   \\ \cline{2-6}
                            & DN4~\cite{li2019revisiting}                    & 86.59 $\pm$ 0.54                            & 82.97 $\pm$ 0.66                           & 59.57 $\pm$ 0.87                   & 64.02 $\pm$ 0.92                            \\
                            & BSNet (D\&C)                               & \textbf{86.68 $\pm$ 0.54}                   & \textbf{84.18 $\pm$ 0.64}                  & \textbf{61.41 $\pm$ 0.92}          & \textbf{65.20 $\pm$ 0.92}                   \\ \bottomrule 
\multirow{8}{*}{Conv6}      & Matching Network~\cite{vinyals2016matching}    & 71.65 $\pm$ 0.72                            & 76.36 $\pm$ 0.60                           & 57.00 $\pm$ 0.94                   & 61.05 $\pm$ 0.93                            \\
                            & BSNet (M\&C)                               & \textbf{74.50 $\pm$ 0.75}                   & \textbf{78.99 $\pm$ 0.65}                  & \textbf{58.32 $\pm$ 0.95}          & \textbf{67.60 $\pm$ 0.91}                   \\ \cline{2-6}
                            & Prototype Network~\cite{snell2017prototypical} & 74.55 $\pm$ 0.71                            & \textbf{80.59 $\pm$ 0.61}                  & 53.92 $\pm$ 0.93                   & 63.63 $\pm$ 0.93                            \\
                            & BSNet (P\&C)                                & \textbf{74.61 $\pm$ 0.71}                   & 80.51 $\pm$ 0.62                           & \textbf{54.25 $\pm$ 0.96}          & \textbf{64.93 $\pm$ 0.99}                   \\ \cline{2-6}
                            & Relation Network~\cite{sung2018learning}       & 72.86 $\pm$ 0.73                            & 79.16 $\pm$ 0.68                           & 47.36 $\pm$ 0.93                   & 62.93 $\pm$ 0.97                            \\
                            & BSNet ( R\&C)                   & \textbf{76.90 $\pm$ 0.70}                   & \textbf{80.32 $\pm$ 0.62}                  & \textbf{55.27 $\pm$ 1.00}          & \textbf{66.19 $\pm$ 0.98}                   \\ \cline{2-6}
                            & DN4~\cite{li2019revisiting}                    & \textbf{87.04 $\pm$ 0.54}                   & \textbf{83.81 $\pm$ 0.65}                  & 42.77 $\pm$ 0.81                   & 66.13 $\pm$ 0.94                            \\
                            & BSNet (D\&C)                                & 86.17 $\pm$ 0.57                            & 83.43 $\pm$ 0.64                           & \textbf{57.16 $\pm$ 0.97}          & \textbf{67.58 $\pm$ 0.95}                   \\ \bottomrule 
\multirow{8}{*}{Conv8}      & Matching Network~\cite{vinyals2016matching}    & 71.84 $\pm$ 0.73                            & 74.84 $\pm$ 0.62                           & 57.41 $\pm$ 0.96                   & 63.44 $\pm$ 0.95                            \\
                            & BSNet (M\&C)                                & \textbf{72.26 $\pm$ 0.72}                   & \textbf{76.96 $\pm$ 0.63}                  & \textbf{57.86 $\pm$ 0.91}          & \textbf{66.06 $\pm$ 0.98}                   \\ \cline{2-6}
                            & Prototype Network~\cite{snell2017prototypical} & \textbf{76.69 $\pm$ 0.68}                   & 81.44 $\pm$ 0.61                           & \textbf{57.66 $\pm$ 0.99}          & \textbf{65.54 $\pm$ 0.97}                   \\
                            & BSNet (P\&C)                                & 76.30 $\pm$ 0.70                            & \textbf{81.77 $\pm$ 0.61}                  & 53.19 $\pm$ 0.96                   & 65.50 $\pm$ 0.99                            \\ \cline{2-6}
                            & Relation Network~\cite{sung2018learning}       & 73.09 $\pm$ 0.75                            & 78.63 $\pm$ 0.66                           & 47.45 $\pm$ 0.94                   & 62.47 $\pm$ 0.99                            \\
                            & BSNet (R\&C)                   & \textbf{75.29 $\pm$ 0.72}                   & \textbf{80.60 $\pm$ 0.63}                  & \textbf{53.77 $\pm$ 0.94}          & \textbf{63.23 $\pm$ 1.01}                   \\ \cline{2-6}
                            & DN4~\cite{li2019revisiting}                    & \textbf{88.14 $\pm$ 0.50}                   & \textbf{84.41 $\pm$ 0.63}                  & \textbf{61.73 $\pm$ 0.97}          & 64.26 $\pm$ 1.01                            \\
                            & BSNet (D\&C)                               & 84.66 $\pm$ 0.59                            & 81.06 $\pm$ 0.72                           & 61.20 $\pm$ 0.97                   & \textbf{64.28 $\pm$ 1.01}                   \\ \bottomrule 

\multirow{2}{*}{ResNet-10}     & Relation Network~\cite{sung2018learning} & 83.85 $\pm$ 0.64          & 81.67 $\pm$ 0.58          & 63.55 $\pm$ 1.04          & \textbf{69.95 $\pm$ 0.97} \\
                               & BSNet (R\&C)                      & \textbf{84.09 $\pm$ 0.66} & \textbf{82.85 $\pm$ 0.61} & \textbf{66.97 $\pm$ 0.99} & 69.73 $\pm$ 0.97          \\ \bottomrule 

\multirow{2}{*}{ResNet-18}     & Relation Network~\cite{sung2018learning} & 83.61 $\pm$ 0.68          & 83.04 $\pm$ 0.60          & \textbf{61.02 $\pm$ 1.04} & 69.58 $\pm$ 0.97          \\ 
                               & BSNet (R\&C)                      & \textbf{85.28 $\pm$ 0.64} & \textbf{83.24 $\pm$ 0.60} & 60.36 $\pm$ 0.98          & \textbf{69.61 $\pm$ 0.92} \\ \bottomrule 

\multirow{2}{*}{ResNet-34} &Relation Network~\cite{sung2018learning} & \textbf{89.07 $\pm$  0.56} & \textbf{84.21 $\pm$ 0.55} & \textbf{76.35 $\pm$ 1.00} & \textbf{71.89 $\pm$ 1.04} \\ 
                               & BSNet (R\&C)                      & 85.30 $\pm$ 0.68 & 82.12 $\pm$ 0.59          & 72.72 $\pm$ 1.02          & 68.98 $\pm$ 1.04          \\ \bottomrule 

\end{tabular}
\label{tab:backbone}
\end{table*}

\textbf{Cosine module:} The cosine module (i.e. the cosine similarity branch in Section~\ref{sec:details_of_architecture}) consists of two convolution blocks and a cosine similarity measurement layer. Each convolution block is composed of $3 \times 3$ convolution of 64 filters, followed by batch normalization, a \textit{ReLU} non-linearity activation function and a pooling layer. The padding of each block is 1. 
The first convolution block is followed by a $2 \times 2$ max-poling, and the second convolution block is followed by a $2 \times 2$ avg-pool. 
Our cosine module does not require to concatenate representations of a query image and class prototypes, and feature maps of the query image and class prototypes are fed independently to the convolution blocks. Then these new feature maps are flattened and pass through a cosine similarity layer to compute the cosine distance between the query image and each class. Cosine similarity scores represent the similarity between query samples and each class's prototype given support samples. For 5-shot learning, a class prototype is the mean of the feature maps of 5 support images from an identical class, while it is directly the feature map of a support image in 1-shot scenarios.

\textbf{Bi-similarity module:}
We combine the self-proposed cosine module with other modules to build our \textit{bi-similarity module}. That is, $h_{\gamma}$ is our cosine module, and $g_{\varphi}$ can be any other modules mentioned above. Note that the parameter setting of \textit{BSNet}'s experiments is the same as that of $g_{\varphi}$ similarity experiments. In this way, features from the embedded module $f_{\phi}$ enter into $h_{\gamma}$ and $g_{\varphi}$ respectively, producing bi-similarity prediction values. At the meta-training stage, we update our network with the average of the prediction losses of the two similarity branches (Equation (\ref{eq:bi_sim_nets_trainingLoss})). At the meta-testing stage, we use the average of two similarity scores generated by the two branches to produce the final prediction (Equation (\ref{eq:gen_one_hot_vec_y_hat_testing})).

\begin{figure*}[h]
\centering
\subfigure{
\begin{minipage}[t]{3.4in}
\centering
\includegraphics[width=3.2in]{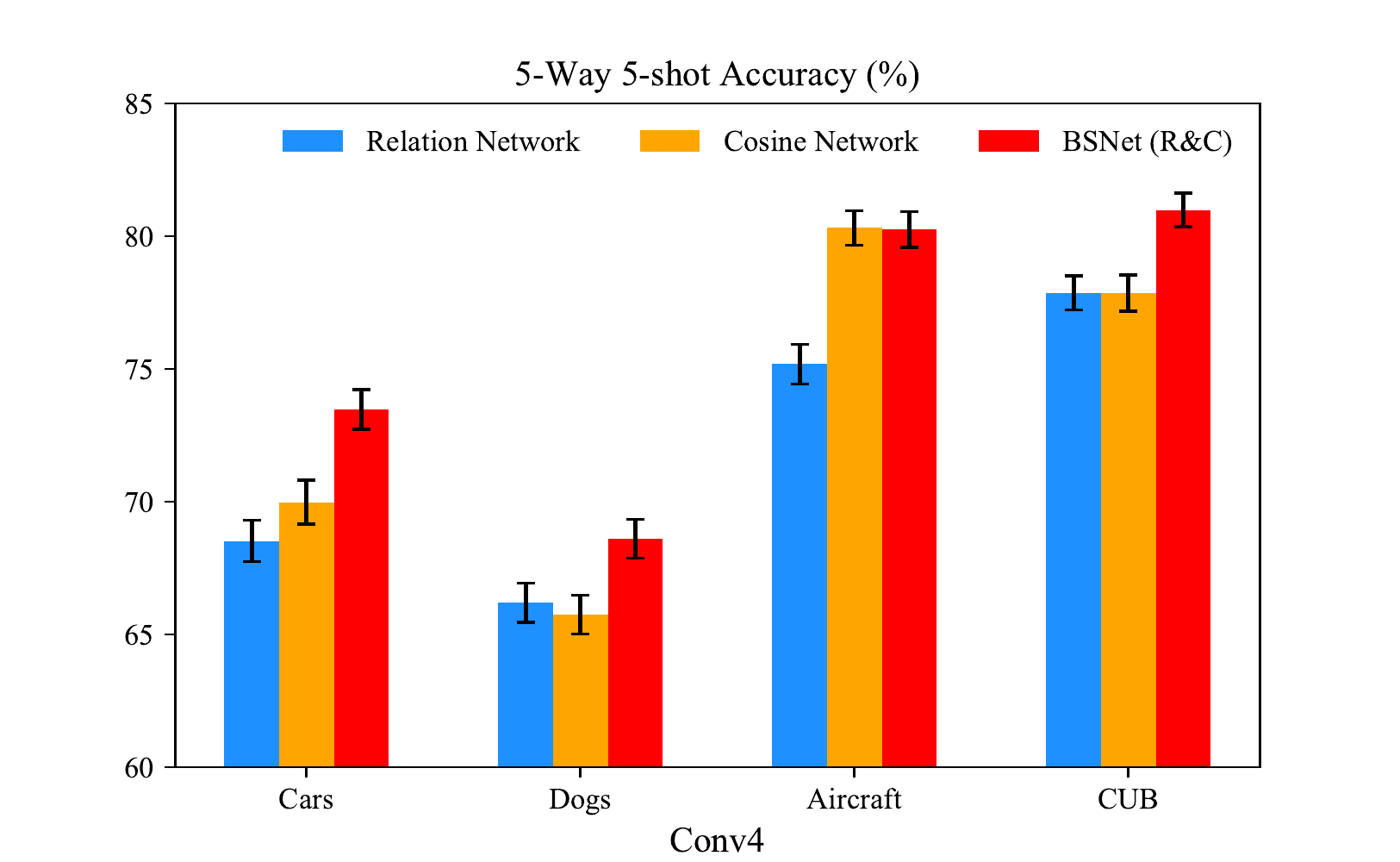}
\end{minipage}%
}%
\subfigure{
\begin{minipage}[t]{3.4in}
\centering
\includegraphics[width=3.2in]{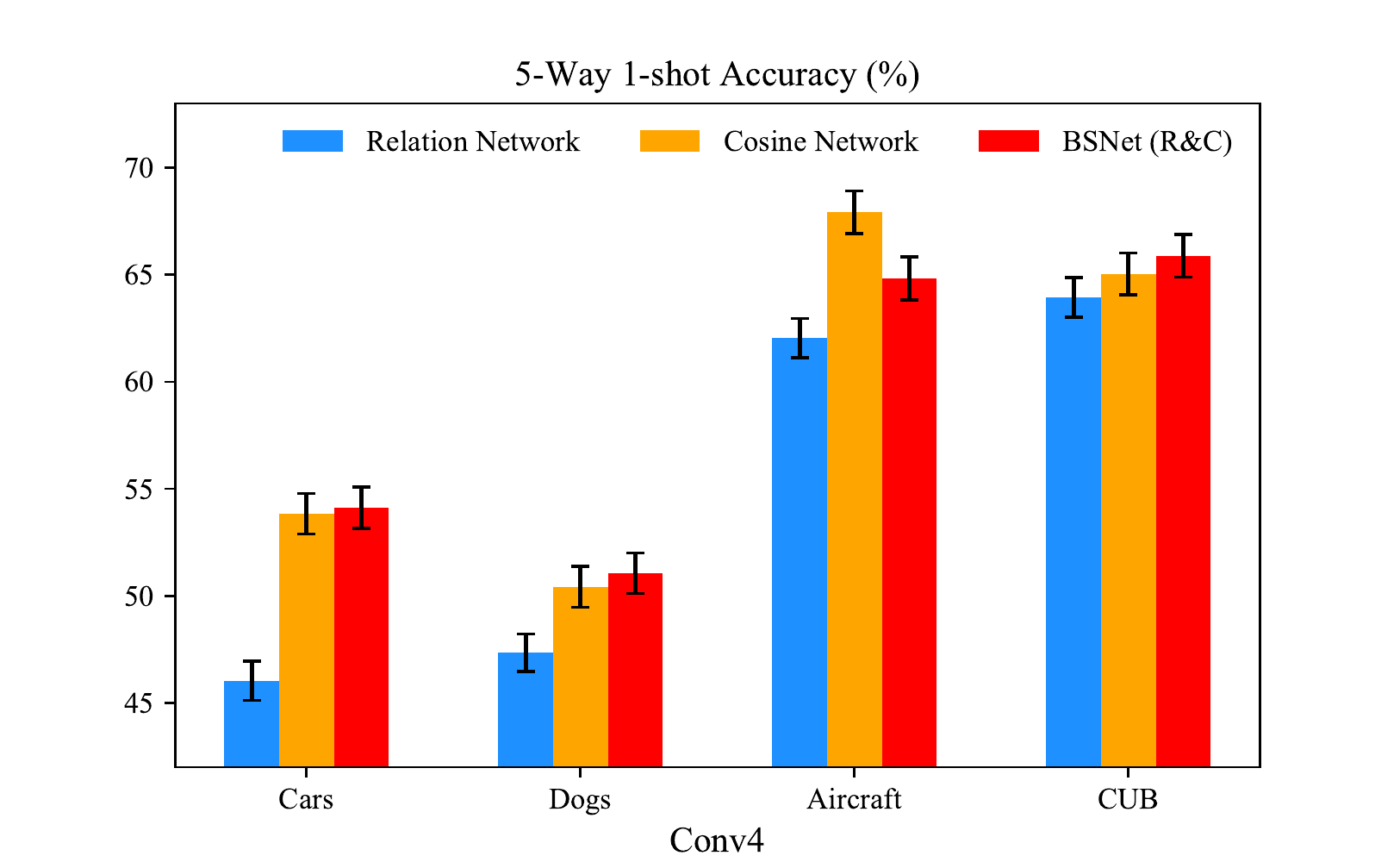}
\end{minipage}%
}%

\subfigure{
\begin{minipage}[t]{3.4in}
\centering
\includegraphics[width=3.0in]{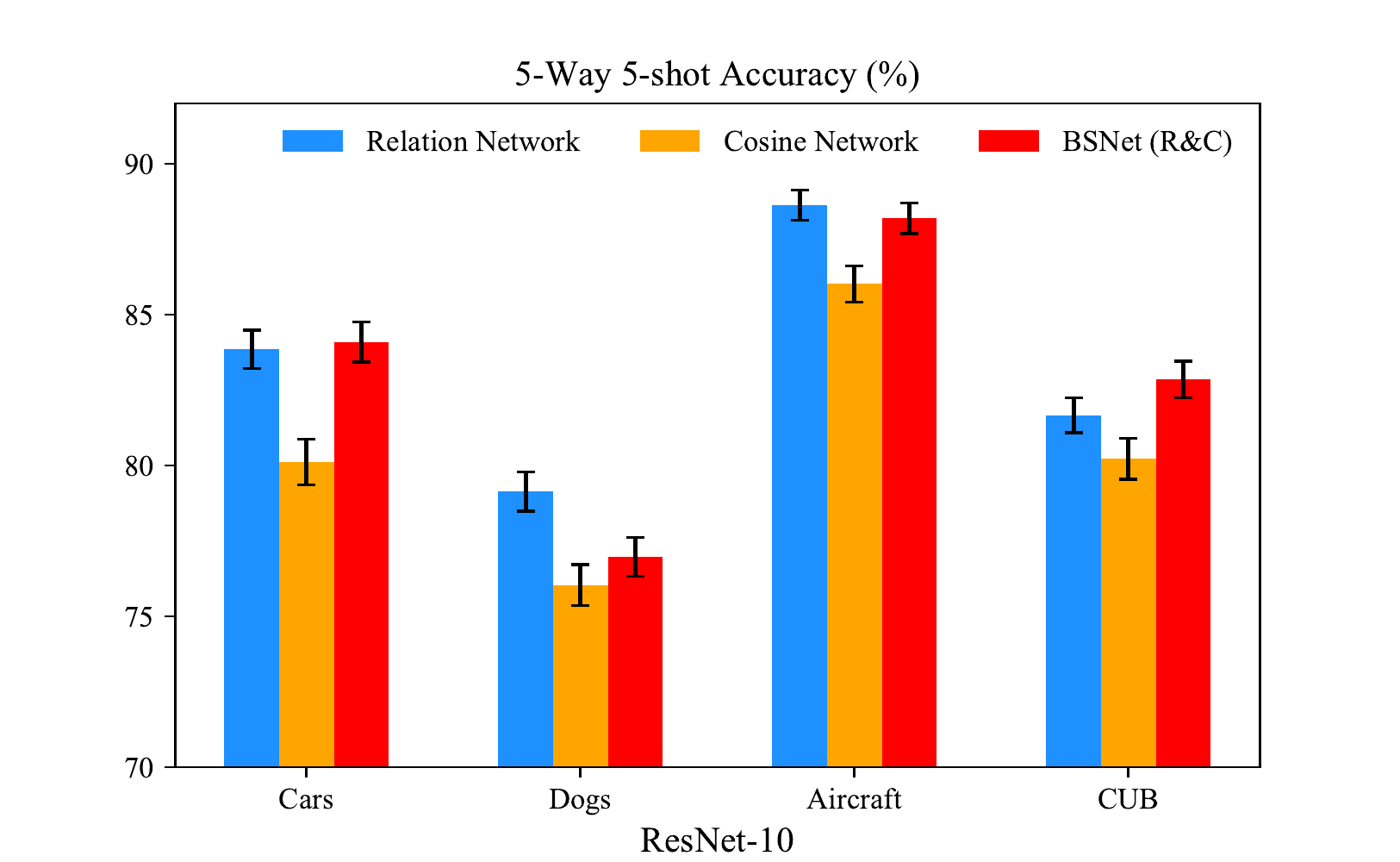}
\end{minipage}
}%
\subfigure{
\begin{minipage}[t]{3.4in}
\centering
\includegraphics[width=3.0in]{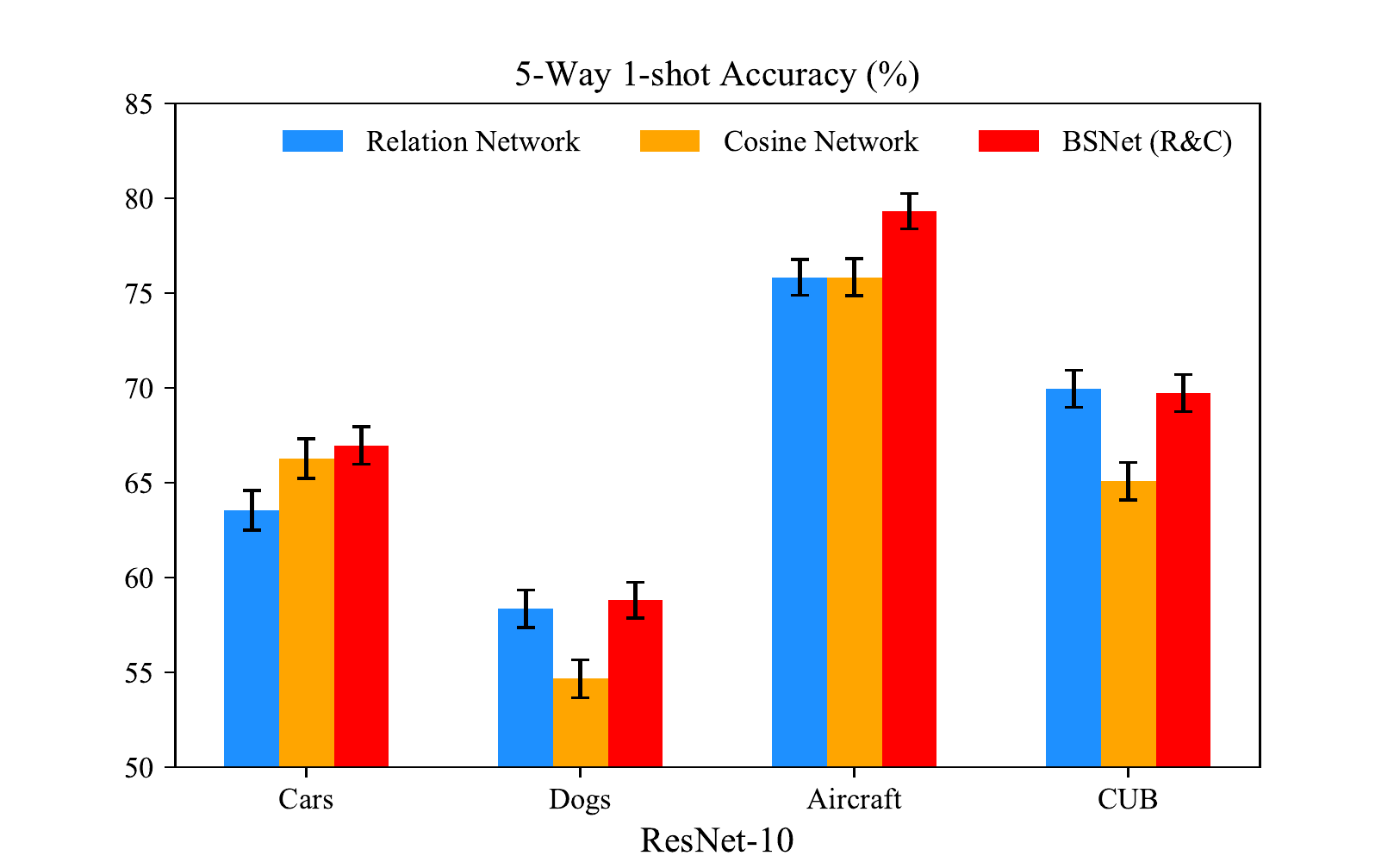}
\end{minipage}
}%

\caption{Ablation study: effectiveness of the two-branch similarity. Five-way few-shot classification performance of \textit{Relation Network}, \textit{Cosine Network} and the proposed \textit{BSNet} (R\&C), on the \textit{Stanford-Cars} (Cars), \textit{Stanford-Dogs} (Dogs), \textit{FGVC-Aircraft} (Aircraft) and \textit{CUB-200-2011} (CUB) datasets. The mean accuracy of 600 randomly generated testing episodes for each method is reported.}
\label{fig:ablation}
\end{figure*}

\subsection{Comparison with State-of-the-art Methods on Few Shot Classification}
\label{sec:sota1}

To evaluate the performance of the proposed \textit{BSNet}, we compare it with several metric/similarity based few-shot classification networks, i.e., \textit{Matching Network}~\cite{vinyals2016matching}, 
\textit{Prototype Network}~\cite{snell2017prototypical}, \textit{Relation Network}~\cite{sung2018learning}, \textit{DN4}~\cite{li2019revisiting}, on the four images datasets mentioned in Sec.~\ref{sec:datasets}. Few shot classification results are shown in Table~\ref{tab:5way_all4metrics}. \textit{BSNet} is implemented by keeping the $h_{\gamma}$ branch of Bi-similarity module to be our cosine module, replacing the $g_{\varphi}$ branch by an arbitrary module among Matching module, Prototype module, Relation module, and Image-to-Class module in Sec.~\ref{sec:Implementation Details}, resulting in \textit{BSNet} (M\&C), \textit{BSNet} (P\&C), \textit{BSNet} (R\&C), and \textit{BSNet} (D\&C) in Table~\ref{tab:5way_all4metrics}, respectively.

Our \textit{BSNet} outperforms \textit{Matching Network} on \textit{Stanford-Dogs}, \textit{FGVC-Aircraft} and \textit{CUB-200-2011}, and achieves comparable generalization performance on \textit{Stanford-Cars} ($63.58\%$ versus $64.74\%$). To compare with \textit{DN4}, we combine its similarity measurement (i.e. the Image-to-Class module in Sec.~\ref{sec:Implementation Details}) with our cosine similarity measurement (i.e. the Cosine module in Sec.~\ref{sec:Implementation Details}), which achieves notable $85.39\%$ on the \textit{CUB-200-2011} dataset. It also demonstrates that the proposed \textit{BSNet} consistently exceeds the \textit{Prototype Network} and \textit{Relation Network} on all of the four fine-grained datasets.

A tougher problem is fine-grained 1-shot classification. We also conducted 5-way 1-shot experiments on the four fine-grained datasets. Table~\ref{tab:5way_all4metrics} demonstrates that the proposed \textit{BSNet} achieves the state-of-the-art classification performance on fine-grained datasets in 1-shot scenarios. That is, the proposed \textit{BSNet} achieves the best classification accuracy on each dataset, $54.12\%$, $51.06\%$, $64.83\%$ and $65.89\%$, respectively. 

It is found that the proposed \textit{BSNet} has great improvements by combining the deep relation measurement (i.e.~the Relation module in Sec.~\ref{sec:Implementation Details}) and our Cosine module, which consistently obtains state-of-the-art performance on all of the four fine-grained datasets in 1-shot scenarios. It is also found that the combination of our Cosine module with other few-shot networks boosts the performance (around $2\%$ to $8\%$ in mean accuracy) with the exception of \textit{Matching Network}.

For the two addends in the loss function of \textit{BSNet} (Equation~\ref{eq:bi_sim_nets_trainingLoss}), we also tried to tune their weights according to  $l_q =  \lambda \times l_q^{(1)} + \beta \times l_q^{(2)}$. Table~\ref{tab:tuneweights} lists the classification performance of \textit{BSNet} (R\&C) on the \textit{CUB-200-2011} dataset when $\lambda$ and $\beta$ are set to different values. Firstly, the highest mean value of the classification accuracies is achieved when the weights of the two addends both equal $1$. Meanwhile, we can only observe tiny difference of the standard deviation between different settings of the weights. Thus, we simply select equal weights for the two addends in the loss function.

\subsection{Few-shot Classification with Different Backbone Networks}\label{sec:Backbones}

In the previous experiments, we used feature embedding modules \textit{Conv4} for the \textit{Matching Network}, \textit{Prototype Network} and \textit{Relation Network}, and \textit{Con64F} for \textit{DN4}. To further investigate the effectiveness of the proposed \textit{BSNet}, We changed the feature embedding module to \textit{Conv4}, \textit{Conv6}, and \textit{Conv8}. We run all the compared methods and \textit{BSNet} on the \textit{Stanford-Cars} and \textit{CUB-200-2011} datasets, respectively. Regarding the structure of \textit{Conv4} (introduced in Section~\ref{sec:Implementation Details}), \textit{Conv6}, and \textit{Conv8}, we used the same setting as~\cite{chen2019a}. The classification results are listed in Table~\ref{tab:backbone}, from which we can make the following observations:

Firstly, when all the methods use \textit{Conv4} as the feature embedding module, for the 5-way 5-shot classification on the \textit{Stanford-Cars} dataset, the proposed \textit{BSNet} is inferior to the \textit{Matching Network}. But in other cases, the proposed method performs better than the \textit{Matching Network}, \textit{Prototype Network}, \textit{Relation Network}, and \textit{DN4} on the \textit{Stanford-Cars} and \textit{CUB-200-2011} datasets in both the 5-way 5-shot classification and the 5-way 1-shot classification scenarios. When all the methods use \textit{Conv6} as the feature embedding module, the proposed \textit{BSNet} performs slightly worse than the \textit{Prototype network} and \textit{DN4} on \textit{CUB-200-2011} dataset, and performs slightly worse than \textit{DN4} on \textit{Stanford-Cars} in the 5-way 5-shot classification. In other cases, \textit{BSNet} performs better than all the compared methods in the 5-way 5-shot classification. Moreover, \textit{BSNet} performs the best in the 5-way 1-shot tasks. Secondly, when all the methods use \textit{Conv8} as the feature embedding module, the proposed \textit{BSNet} performs worse than Prototype Network and \textit{DN4} in 5-way 1-shot and 5-way 5-shot tasks, respectively. However, it performs better than \textit{Matching Network} and \textit{Relation Network} on \textit{Stanford-Cars} and \textit{CUB-200-2011} for both the 5-way 5-shot and 5-way 1-shot tasks.

In addition to traditional convolutional backbone networks, we test feature embedding by \textit{ResNet-10}, \textit{ResNet-18} or \textit{ResNet-34}, for \textit{Relation Network} and \textit{BSNet} on the \textit{Stanford-Cars} and \textit{CUB-200-2011} datasets (see Table~\ref{tab:backbone}). We used the same setting as~\cite{chen2019a}, except that the epoch number of \textit{ResNet-34} is set to 1,800 since this feature embedding network is difficult to converge on the training data, especially for \textit{BSNet}. From Table~\ref{tab:backbone}, we find that, when \textit{Relation Network} and \textit{BSNet} use \textit{ResNet-10} or \textit{ResNet-18} as the feature embedding module, 

\textit{BSNet} generally outperforms \textit{Relation Network}.
When \textit{ResNet-34} is chosen, however, \textit{Relation Network} is better.

In brief, we can conclude that, firstly, when we change the feature embedding to \textit{Conv4}, \textit{Conv6}, \textit{Conv8}, \textit{ResNet-10}, and \textit{ResNet-18}, \textit{BSNet} outperforms the compared methods in most of the cases. Secondly, the performance of all compared methods will increase when the embedding layer becomes deeper, which is consistent with the findings in \cite{chen2019a}.

\begin{figure*}[h]
\centering
\includegraphics[width=6.5in]{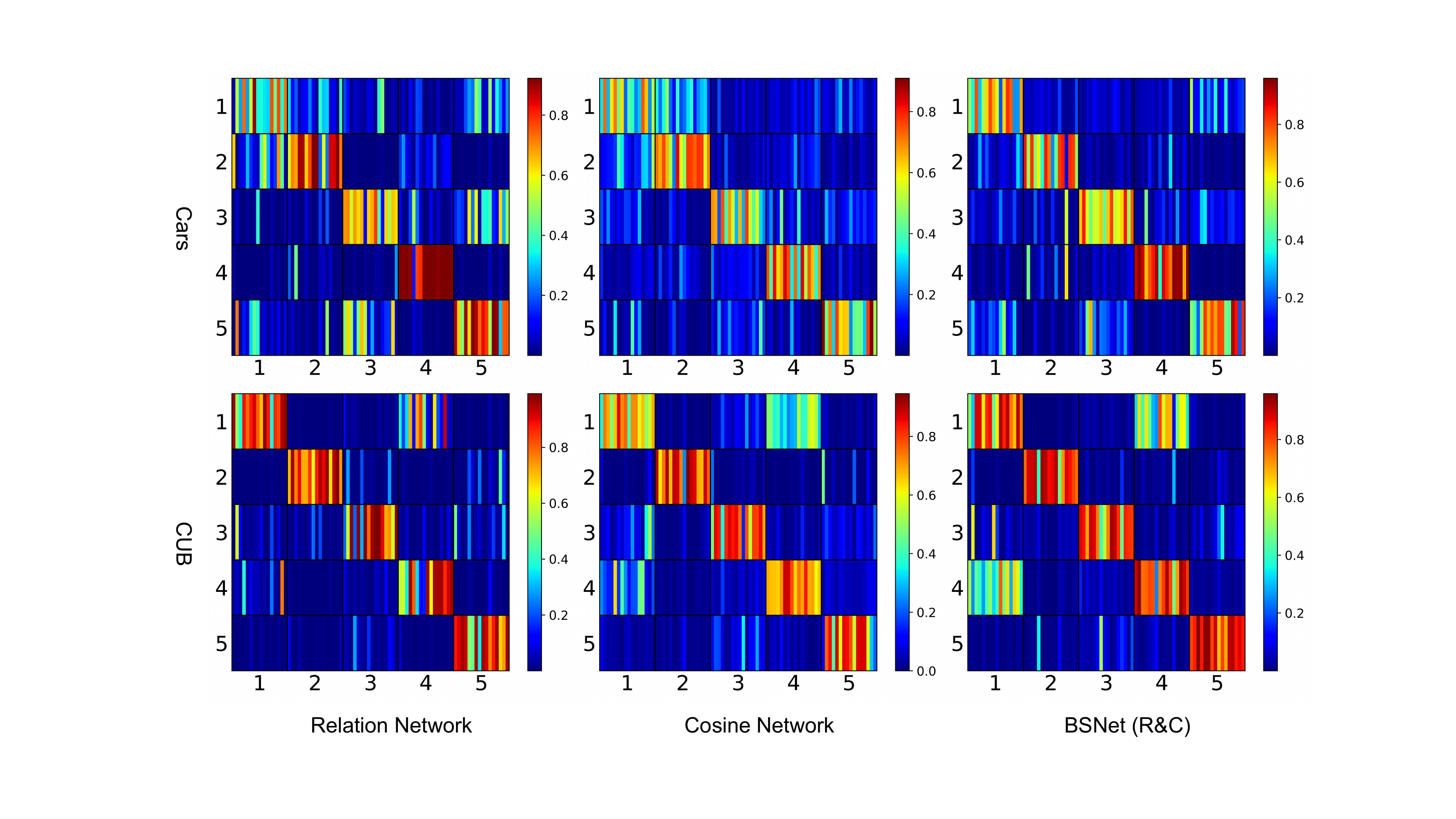}
\caption{Visualization of similarity scores predicted by \textit{Relation Network}~\cite{sung2018learning}, \textit{Cosine Network} and the proposed \textit{BSNet} on the \textit{Stanford-Cars} (Cars) and \textit{CUB-200-2011} (CUB) datasets. In each confusion matrix, the vertical axis shows 5 classes in a task, the horizontal axis shows query samples in the 5 classes, each class contains 16 query samples. Warmer color means larger similarity score.}
\label{fig:scores_heatmap_5_shot}
\end{figure*}

\begin{table}[!t]
\renewcommand
\arraystretch{1.3}
\centering

\caption{Five-way few-shot classification performance on the \textit{Stanford-Cars} (Cars) and \textit{CUB-200-2011} (CUB) datasets. The methods include: \textit{Matching Network}, \textit{Prototype Network}, \textit{Relation Network}, and the proposed \textit{BSNet} but with two other similarity metrics. We report the mean accuracy for each method, along with its $95\%$ confidence interval.}
\begin{tabular}{ccll}
\toprule[1pt]

\multirow{2}{*}{\it{Shot}}   & \multirow{2}{*}{\it{Model}}            & \multicolumn{2}{c}{5-Way Accuracy (\%)}                         \\ \cmidrule(r){3-4} 
                             &                                        & \multicolumn{1}{c}{\it{Cars}} & \multicolumn{1}{c}{\it{CUB}}    \\ \toprule
\multirow{6}{*}{\it{5-shot}} & Matching~\cite{vinyals2016matching}    & 64.74 $\pm$ 0.72              & 74.57 $\pm$ 0.73                \\
                             & Prototype~\cite{snell2017prototypical} & 62.14 $\pm$ 0.76              & 75.06 $\pm$ 0.67                \\ 
                             & Relation~\cite{sung2018learning}       & 68.52 $\pm$ 0.78              & 77.87 $\pm$ 0.64                \\ \cmidrule(r){2-4} 
                             & BSNet (M\&P)                           & \textbf{68.69 $\pm$ 0.70}     & \textbf{78.20 $\pm$ 0.64}       \\ 
                             & BSNet (M\&R)                           & 67.32 $\pm$ 0.71              & 77.95 $\pm$ 0.69                \\ 
                             & BSNet (P\&R)                           & 63.73 $\pm$ 0.74              & 76.83 $\pm$ 0.65                \\ \bottomrule

\multirow{6}{*}{\it{1-shot}} & Matching~\cite{vinyals2016matching}    & 44.73 $\pm$ 0.77              & 60.06 $\pm$ 0.88          \\
                             & Prototype~\cite{snell2017prototypical} & 36.54 $\pm$ 0.74              & 50.67 $\pm$ 0.88          \\ 
                             & Relation~\cite{sung2018learning}       & \textbf{46.04 $\pm$ 0.91}     & \textbf{63.94 $\pm$ 0.92} \\ \cmidrule(r){2-4} 
                             & BSNet (M\&P)                           & 38.44 $\pm$ 0.75              & 52.10 $\pm$ 0.90          \\ 
                             & BSNet (M\&R)                           & 45.92 $\pm$ 0.79              & 62.70 $\pm$ 0.92          \\ 
                             & BSNet (P\&R)                           & 41.20 $\pm$ 0.83              & 52.95 $\pm$ 0.98          \\ \bottomrule
\end{tabular}

\label{tab:anytwometrics}
\end{table}

\begin{table*}[htp!]

\renewcommand
\arraystretch{1.3}
\centering

\caption{Five-way few-shot classification performance on the \textit{Stanford-Cars} and \textit{CUB-200-2011} datasets. The methods include: \textit{Matching Network}, \textit{Prototype Network}, \textit{Relation Network}, \textit{Cosine Network}, and the proposed \textit{BSNet} with three or four similarity metrics. We report the mean accuracy for each method, along with its $95\%$ confidence interval.}
\label{tab:4metrics}
\begin{tabular}{cllll}
\toprule[1pt]
\multirow{2}{*}{\it{Model}}            & \multicolumn{2}{c}{5-Way 5-shot Accuracy (\%)}                                                                                       & \multicolumn{2}{c}{5-Way 1-shot Accuracy (\%)}                        \\  \cmidrule(r){2-3}  \cmidrule(r){4-5} 
                                       & \multicolumn{1}{c}{\it{Stanford-Cars}}& \multicolumn{1}{c}{\it{CUB-200-2011}} & \multicolumn{1}{c}{\it{Stanford-Cars}} & \multicolumn{1}{c}{\it{CUB-200-2011}}\\ \toprule

Matching Network~\cite{vinyals2016matching}    & 64.74 $\pm$ 0.72          & 74.57 $\pm$ 0.73          & 44.73 $\pm$ 0.77          & 60.06 $\pm$ 0.88          \\
Prototype Network~\cite{snell2017prototypical} & 62.14 $\pm$ 0.76          & 75.06 $\pm$ 0.67          & 36.54 $\pm$ 0.74          & 50.67 $\pm$ 0.88          \\ 
Relation Network~\cite{sung2018learning}       & 68.52 $\pm$ 0.78          & 77.87 $\pm$ 0.64          & 46.04 $\pm$ 0.91          & 63.94 $\pm$ 0.92          \\ 
Cosine Network                                 & 69.98 $\pm$ 0.83          & 77.86 $\pm$ 0.68          & 53.84 $\pm$ 0.94          & 65.04 $\pm$ 0.97          \\ \bottomrule

BSNet (M\&P\&R)                        & 71.53 $\pm$ 0.73          & 79.83 $\pm$ 0.63          & 45.56 $\pm$ 0.83          & 60.28 $\pm$ 0.94          \\ 
BSNet (M\&P\&C)                        & 71.50 $\pm$ 0.75          & 79.30 $\pm$ 0.61          & 44.33 $\pm$ 0.83          & 59.18 $\pm$ 0.93          \\ 
BSNet (M\&R\&C)                        & 72.70 $\pm$ 0.71          & 80.84 $\pm$ 0.67          & \textbf{54.82 $\pm$ 0.89} & \textbf{66.13 $\pm$ 0.90} \\ 
BSNet (P\&R\&C)                        & 69.20 $\pm$ 0.76          & 79.26 $\pm$ 0.62          & 47.09 $\pm$ 0.85          & 58.98 $\pm$ 0.96          \\ \bottomrule
BSNet (M\&P\&R\&C)                     & \textbf{72.78 $\pm$ 0.73} & \textbf{80.94 $\pm$ 0.63} & 48.00 $\pm$ 0.87          & 59.50 $\pm$ 0.96          \\ \bottomrule
\end{tabular}

\end{table*}

\begin{figure*}[h]
\centering
\includegraphics[width=6.5in]{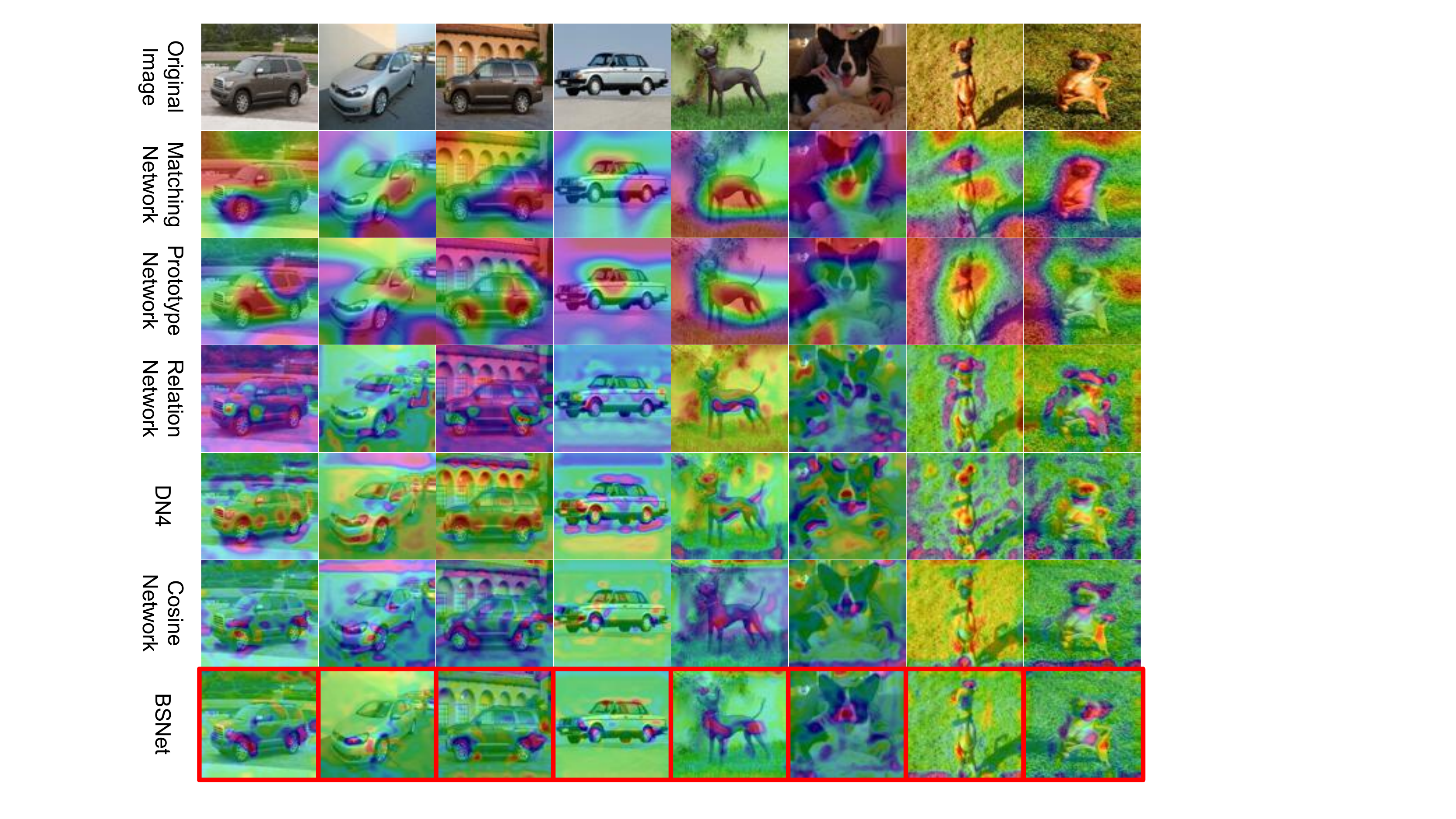}
\caption{Feature visualization under \textit{Matching Network}~\cite{vinyals2016matching}, \textit{Prototype Network}~\cite{snell2017prototypical}, \textit{Relation Network}~\cite{sung2018learning}, \textit{DN4}~\cite{li2019revisiting}, \textit{Cosine Network} and the proposed \textit{BSNet} on the \textit{Stanford-Cars} and \textit{Stanford-Dogs} datasets. The redder the region, the more class-discriminative it is.}
\label{fig:heatmap_5_shot}
\end{figure*}

\subsection{Ablation Study on Effectiveness of Bi-Similarity Module}
\label{sec:on Effectiveness of of Two-branch}

To further explore the effect of Bi-similarity module, in this section, we prune either of two similarity branches in Bi-similarity module. If only keeping Relation similarity branch and pruning the Cosine similarity branch, we recover the \textit{Relation Network}~\cite{sung2018learning}. Similarly, if only keeping our cosine branch, we obtain a single Cosine module based network, which is denoted by \textit{Cosine Network} in this work. 

We compare the performance of \textit{Relation Network}, \textit{Cosine Network} and the proposed \textit{BSNet} with different embedding modules here, i.e., \textit{Conv4} and \textit{Resnet-10}, on the four fine-grained datasets. Experimental results of 5-way 5-shot and 5-way 1-shot tasks are presented in Figure~\ref{fig:ablation}. From Figure~\ref{fig:ablation}, it can be found that firstly, in some cases \textit{Relation Network} performs better than \textit{Cosine Network}, and in other cases \textit{Cosine Network} performs better than \textit{Relation Network}. Secondly, in most cases, the proposed \textit{BSNet} outperforms both of the \textit{Relation Network} and the \textit{Cosine Network}.

We visualize similarity scores obtained by \textit{Relation Network}, \textit{Cosine Network}, and the proposed \textit{BSNet} in the 5-way 5-shot experiments. In particular, we fixed the testing tasks from the meta-testing set of the \textit{Stanford-Cars} and \textit{CUB-200-2011} datasets, respectively, so that the sequences of testing tasks for different methods on each dataset are the same. On the \textit{Stanford-Cars} dataset, the prediction accuracy of \textit{Relation Network}, \textit{Cosine Network} and the proposed \textit{BSNet} are $68.52 \pm 0.78\%$, $69.98 \pm 0.83\%$ and $73.47 \pm 0.75\%$, respectively. On the \textit{CUB-200-2011} dataset, the prediction accuracy of \textit{Relation Network}, \textit{Cosine Network} and the proposed \textit{BSNet} are $77.87 \pm 0.64\%$, $77.86 \pm 0.68\%$ and $80.99 \pm 0.63\%$, respectively.
For each dataset, we randomly select a testing task and show the similarity scores of its query images via confusion matrices. Please refer to Figure \ref{fig:scores_heatmap_5_shot} for details.

From Figure \ref{fig:scores_heatmap_5_shot}, it can be observed that on the \textit{Stanford-Cars} dataset, for the query images from the 4th class, the predicting similarity scores of \textit{Relation Network} are higher than those of \textit{Cosine Network}; on the \textit{CUB-200-2011} dataset, for the query images from the 3rd class, the predicting similarity scores of \textit{Cosine Network} are higher than those of \textit{Relation Network}. However, in both cases, \textit{BSNet} can predict query images correctly by synthesizing the Relation module and the Cosine module. Similar patterns can also be found in other tasks.

These results further show that our bi-similarity idea is reasonable and the proposed \textit{BSNet} is less biased on similarity.

\subsection{Feature Visualization}
\label{sec:Feature Visualization}
To further demonstrate that the features learned by the proposed \textit{BSNet} are distributed in a smaller feature space and are more discriminative, we use a gradient-based technique, \textit{Grad-CAM}~\cite{selvaraju2017gradcam}, to visualize the important regions in the original images, which is illustrated in Figure~\ref{fig:motivation}.

In Figure~\ref{fig:heatmap_5_shot}, we randomly select $8$ images ($4$ from \textit{Stanford-Cars}, $4$ from \textit{Stanford-Dogs}) and resize the original images to the same size as the output of the embedding layer $f_{\phi}$. The resized raw images are compared to the outputs of \textit{Grad-CAM} under the setting of \textit{Matching Network}, \textit{Prototype Network}, \textit{Relation Network}, \textit{DN4 Network}, our \textit{Cosine Network} and the proposed \textit{BSNet}. {Figure~\ref{fig:heatmap_5_shot} shows that in comparison with other compared methods, the proposed \textit{BSNet} consistently has a reduced number of class-discriminative regions concentrated in the regions of  
``cars'' or ``dogs''}, thus the features learned by \textit{BSNet} are more robust and efficient.

\subsection{Effect of Changing Similarity Modules on BSNet}\label{sec:ChangingSimilarityModules}

To further show the applicability of \textit{BSNet}, we implemented \textit{BSNet} based on two similarity metrics among Prototype module, Relation module and Matching module (see Table \ref{tab:anytwometrics}). In addition, we also extended \textit{BSNet} to a network with three or four similarity modules (see Table \ref{tab:4metrics}).

From Table \ref{tab:anytwometrics}, we can observe that, for 5-way 5-shot classification, either \textit{BSNet} is better than single similarity metric networks, e.g., \textit{BSNet(M\&P)} outperforms \textit{Matching Network} and \textit{Prototype Network} on the \textit{Stanford-Cars} dataset, or \textit{BSNet} performs in between the two single similarity metric networks, e.g., \textit{BSNet(M\&R)} underperforms \textit{Matching Network} but outperforms \textit{Relation Network} on the \textit{Stanford-Cars} dataset. For 5-way 1-shot classification, the accuracy of \textit{BSNet} is in between those of the two single similarity metric networks on the two datasets. 

From Table \ref{tab:4metrics}, we can see that, when we extend the proposed \textit{BSNet} to a network with three similarity metrics, it still works well on the \textit{Stanford-Cars} and \textit{CUB-200-2011} datasets. \textit{BSNet} either outperforms all single similarity networks, or performs better than the worst one of three single similarity networks but worse than the best one of three single similarity networks. When we equip \textit{BSNet} with four similarity metrics, similar patterns can be found.

In short, the above results indicate that: firstly, when we replace the similarity branch with different similarity metric modules, the improvement over a single similarity network is consistent with that of the \textit{BSNet} with Relation module and Cosine module. Secondly, when we extend \textit{BSNet} to multiple similarity metrics, the improvement is also consistent with the bi-similarity \textit{BSNet}.

\subsection{Discussion}
\label{sec:Discussion}
The experimental results have demonstrated the effectiveness of \textit{BSNet}. As listed in Table~\ref{tab:5way_all4metrics} and Table~\ref{tab:backbone}, in most cases, the proposed \textit{BSNet} can improve the state-of-the-art metric-based few-shot learning methods. This can be attributed to the following properties of \textit{BSNet}.
Firstly, compared with the single-similarity network, even though the proposed \textit{BSNet} contains more model parameters, it does not necessarily increase the empirical Rademacher complexity according to the Theorem~\ref{theorem}. When a single-similarity based few-shot method has larger model complexity, bi-similarity can reduce the model complexity, and thus resulting in better generalization performance. Therefore, when a single-similarity network is excessively flexible, a \textit{BSNet} with better generalization can be constructed by simply adding an additional similarity module with lower model complexity onto the single-similarity network. Secondly, \textit{BSNet} can learn more discriminative features than other single-similarity networks, as the feature embedding in \textit{BSNet} needs to meet two distinct similarity measures.

The experimental results demonstrate the applicability of \textit{BSNet}. As listed in Table~\ref{tab:anytwometrics} and~\ref{tab:4metrics}, the proposed \textit{BSNet} either performs better than all individual similarity networks, or performs better than the worst one of three individual similarity networks. We can interpret this pattern from the perspective of the Rademacher complexity: the Rademacher complexity of the proposed \textit{BSNet} is no more than the average of the Rademacher complexities of two individual networks. 
Hence, according to the relationship between a model generalization error bound and an empirical Rademacher complexity (Theorem 3.1 in~\cite{mohri2012foundations}), when the Rademacher complexity of the proposed \textit{BSNet} is lower than both two individual networks, the proposed \textit{BSNet} may perform better than the two individual networks; 
otherwise, \textit{BSNet} may perform in between the two individual networks. Thus, in terms of the similarity metric selection for \textit{BSNet}, it is desirable to have each similarity metric module with lower Rademacher complexity.

\section{Conclusion}
\label{sec:Conclusion}
In this paper, we proposed a novel neural network, namely \textit{Bi-Similarity Network}, for few-shot fine-grained image classification. The proposed network contains a shared embedding module and a bi-similarity module, the structure of which contains a parameter sharing mechanism. The {\it Bi-Similarity Network} can learn fewer but more discriminative regions compared with other single metric/similarity based few-shot learning neural networks. Extensive experiments demonstrate that the proposed \textit{BSNet} outperforms or matches previous state-of-the-art performance on fine-grained image datasets.

\bibliographystyle{IEEEtran}
\small
\bibliography{main}

\begin{thebibliography}{10}
\providecommand{\url}[1]{#1}
\csname url@samestyle\endcsname
\providecommand{\newblock}{\relax}
\providecommand{\bibinfo}[2]{#2}
\providecommand{\BIBentrySTDinterwordspacing}{\spaceskip=0pt\relax}
\providecommand{\BIBentryALTinterwordstretchfactor}{4}
\providecommand{\BIBentryALTinterwordspacing}{\spaceskip=\fontdimen2\font plus
\BIBentryALTinterwordstretchfactor\fontdimen3\font minus
  \fontdimen4\font\relax}
\providecommand{\BIBforeignlanguage}[2]{{%
\expandafter\ifx\csname l@#1\endcsname\relax
\typeout{** WARNING: IEEEtran.bst: No hyphenation pattern has been}%
\typeout{** loaded for the language `#1'. Using the pattern for}%
\typeout{** the default language instead.}%
\else
\language=\csname l@#1\endcsname
\fi
#2}}
\providecommand{\BIBdecl}{\relax}
\BIBdecl

\bibitem{lecun2015deep}
Y.~LeCun, Y.~Bengio, and G.~Hinton, ``Deep learning,'' \emph{Nature}, vol. 521,
  no. 7553, p. 436, 2015.

\bibitem{dvornik2019diversity}
N.~Dvornik, C.~Schmid, and J.~Mairal, ``Diversity with cooperation: Ensemble
  methods for few-shot classification,'' in \emph{Proceedings of the IEEE
  International Conference on Computer Vision}, 2019, pp. 3723--3731.

\bibitem{lin2019partition}
W.~Lin, X.~He, X.~Han, D.~Liu, J.~See, J.~Zou, H.~Xiong, and F.~Wu,
  ``Partition-aware adaptive switching neural networks for post-processing in
  hevc,'' \emph{IEEE Transactions on Multimedia}, 2019.

\bibitem{simonyan2014very}
K.~Simonyan and A.~Zisserman, ``Very deep convolutional networks for
  large-scale image recognition,'' \emph{arXiv preprint arXiv:1409.1556}, 2014.

\bibitem{gu2015recent}
J.~Gu, Z.~Wang, J.~Kuen, L.~Ma, A.~Shahroudy, B.~Shuai, T.~Liu, X.~Wang, and
  G.~Wang, ``Recent advances in convolutional neural networks,'' \emph{arXiv
  preprint arXiv:1512.07108}, 2015.

\bibitem{lin2016tube}
W.~Lin, Y.~Zhou, H.~Xu, J.~Yan, M.~Xu, J.~Wu, and Z.~Liu, ``A
  tube-and-droplet-based approach for representing and analyzing motion
  trajectories,'' \emph{IEEE transactions on pattern analysis and machine
  intelligence}, vol.~39, no.~8, pp. 1489--1503, 2016.

\bibitem{li2020oslnet}
X.~{Li}, D.~{Chang}, Z.~{Ma}, Z.~{Tan}, J.~{Xue}, J.~{Cao}, J.~{Yu}, and
  J.~{Guo}, ``Oslnet: Deep small-sample classification with an orthogonal
  softmax layer,'' \emph{IEEE Transactions on Image Processing}, vol.~29, pp.
  6482--6495, 2020.

\bibitem{li2020concise}
X.~Li, Z.~Sun, J.-H. Xue, and Z.~Ma, ``A concise review of recent few-shot
  meta-learning methods,'' \emph{Neurocomputing}, 2020.

\bibitem{li2020remarnet}
X.~Li, L.~Yu, X.~Yang, Z.~Ma, J.-H. Xue, J.~Cao, and J.~Guo, ``Remarnet:
  Conjoint relation and margin learning for small-sample image
  classification,'' \emph{IEEE Transactions on Circuits and Systems for Video
  Technology}, 2020.

\bibitem{lifchitz2019dense}
Y.~Lifchitz, Y.~Avrithis, S.~Picard, and A.~Bursuc, ``Dense classification and
  implanting for few-shot learning,'' in \emph{Proceedings of the IEEE
  Conference on Computer Vision and Pattern Recognition}, 2019, pp. 9258--9267.

\bibitem{sung2018learning}
F.~Sung, Y.~Yang, L.~Zhang, T.~Xiang, P.~H. Torr, and T.~M. Hospedales,
  ``Learning to compare: Relation network for few-shot learning,'' in
  \emph{IEEE Conference on Computer Vision and Pattern Recognition}, 2018, pp.
  1199--1208.

\bibitem{Gidaris_2019_ICCV}
S.~Gidaris, A.~Bursuc, N.~Komodakis, P.~P{\'e}rez, and M.~Cord, ``Boosting
  few-shot visual learning with self-supervision,'' in \emph{IEEE International
  Conference on Computer Vision}, 2019, pp. 8059--8068.

\bibitem{bertinetto2016learning_feedforward_oneshot}
L.~Bertinetto, J.~F. Henriques, J.~Valmadre, P.~Torr, and A.~Vedaldi,
  ``Learning feed-forward one-shot learners,'' in \emph{Advances in Neural
  Information Processing Systems}, 2016, pp. 523--531.

\bibitem{finn2017maml}
C.~Finn, P.~Abbeel, and S.~Levine, ``Model-agnostic meta-learning for fast
  adaptation of deep networks,'' in \emph{International Conference on Machine
  Learning}, 2017, pp. 1126--1135.

\bibitem{finn2018probabilistic_maml}
C.~Finn, K.~Xu, and S.~Levine, ``Probabilistic model-agnostic meta-learning,''
  in \emph{Advances in Neural Information Processing Systems}, 2018, pp.
  9516--9527.

\bibitem{qi2018low_shot_learning_imprintedweights}
H.~Qi, M.~Brown, and D.~G. Lowe, ``Low-shot learning with imprinted weights,''
  in \emph{IEEE Conference on Computer Vision and Pattern Recognition}, 2018,
  pp. 5822--5830.

\bibitem{li2019lgm}
H.~Li, W.~Dong, X.~Mei, C.~Ma, F.~Huang, and B.-G. Hu, ``{LGM-Net}: Learning to
  generate matching networks for few-shot learning,'' \emph{arXiv preprint
  arXiv:1905.06331}, 2019.

\bibitem{santoro2016meta}
A.~Santoro, S.~Bartunov, M.~Botvinick, D.~Wierstra, and T.~Lillicrap,
  ``Meta-learning with memory-augmented neural networks,'' in
  \emph{International conference on machine learning}, 2016, pp. 1842--1850.

\bibitem{munkhdalai2017meta}
T.~Munkhdalai and H.~Yu, ``Meta networks,'' in \emph{Proceedings of the 34th
  International Conference on Machine Learning}, 2017, pp. 2554--2563.

\bibitem{vinyals2016matching}
O.~Vinyals, C.~Blundell, T.~Lillicrap, K.~Kavukcuoglu, and D.~Wierstra,
  ``Matching networks for one shot learning,'' in \emph{Advances in Neural
  Information Processing Systems}, 2016, pp. 3630--3638.

\bibitem{snell2017prototypical}
J.~Snell, K.~Swersky, and R.~Zemel, ``Prototypical networks for few-shot
  learning,'' in \emph{Advances in Neural Information Processing Systems},
  2017, pp. 4077--4087.

\bibitem{li2019revisiting}
W.~Li, L.~Wang, J.~Xu, J.~Huo, Y.~Gao, and J.~Luo, ``Revisiting local
  descriptor based image-to-class measure for few-shot learning,'' in
  \emph{IEEE Conference on Computer Vision and Pattern Recognition}, 2019, pp.
  7260--7268.

\bibitem{chen2019a}
W.-Y. Chen, Y.-C. Liu, Z.~Kira, Y.-C.~F. Wang, and J.-B. Huang, ``A closer look
  at few-shot classification,'' in \emph{International Conference on Learning
  Representations}, 2019.

\bibitem{lin2017learning}
W.~Lin, Y.~Shen, J.~Yan, M.~Xu, J.~Wu, J.~Wang, and K.~Lu, ``Learning
  correspondence structures for person re-identification,'' \emph{IEEE
  Transactions on Image Processing}, vol.~26, no.~5, pp. 2438--2453, 2017.

\bibitem{Qiao_2019_ICCV}
L.~Qiao, Y.~Shi, J.~Li, Y.~Wang, T.~Huang, and Y.~Tian, ``Transductive
  episodic-wise adaptive metric for few-shot learning,'' in \emph{IEEE
  International Conference on Computer Vision}, 2019, pp. 3603--3612.

\bibitem{Tokmakov_2019_ICCV}
P.~Tokmakov, Y.-X. Wang, and M.~Hebert, ``Learning compositional
  representations for few-shot recognition,'' in \emph{IEEE International
  Conference on Computer Vision}, 2019, pp. 6372--6381.

\bibitem{Hao_2019_ICCV}
F.~Hao, F.~He, J.~Cheng, L.~Wang, J.~Cao, and D.~Tao, ``Collect and select:
  Semantic alignment metric learning for few-shot learning,'' in \emph{IEEE
  International Conference on Computer Vision}, October 2019.

\bibitem{li2019TVT}
X.~Li, L.~Yu, D.~Chang, Z.~Ma, and J.~Cao, ``Dual cross-entropy loss for
  small-sample fine-grained vehicle classification,'' \emph{IEEE Transactions
  on Vehicular Technology}, vol.~68, no.~5, pp. 4204--4212, 2019.

\bibitem{zhang2016picking}
X.~Zhang, H.~Xiong, W.~Zhou, W.~Lin, and Q.~Tian, ``Picking deep filter
  responses for fine-grained image recognition,'' in \emph{Proceedings of the
  IEEE conference on computer vision and pattern recognition}, 2016, pp.
  1134--1142.

\bibitem{ma2019fine}
Z.~Ma, D.~Chang, J.~Xie, Y.~Ding, S.~Wen, X.~Li, Z.~Si, and J.~Guo,
  ``Fine-grained vehicle classification with channel max pooling modified
  cnns,'' \emph{IEEE Transactions on Vehicular Technology}, vol.~68, no.~4, pp.
  3224--3233, 2019.

\bibitem{7937818}
X.~{Zhang}, H.~{Xiong}, W.~{Zhou}, W.~{Lin}, and Q.~{Tian}, ``Picking neural
  activations for fine-grained recognition,'' \emph{IEEE Transactions on
  Multimedia}, vol.~19, no.~12, pp. 2736--2750, 2017.

\bibitem{chang2020devil}
D.~Chang, Y.~Ding, J.~Xie, A.~K. Bhunia, X.~Li, Z.~Ma, M.~Wu, J.~Guo, and Y.-Z.
  Song, ``The devil is in the channels: Mutual-channel loss for fine-grained
  image classification,'' \emph{IEEE Transactions on Image Processing},
  vol.~29, pp. 4683--4695, 2020.

\bibitem{koch2015siamese}
G.~Koch, R.~Zemel, and R.~Salakhutdinov, ``Siamese neural networks for one-shot
  image recognition,'' in \emph{ICML deep learning workshop}, vol.~2, 2015.

\bibitem{VGG_oxford}
K.~Simonyan and A.~Zisserman, ``Very deep convolutional networks for
  large-scale image recognition,'' in \emph{International Conference on
  Learning Representations}, 2015.

\bibitem{allen2019infinite}
K.~R. Allen, E.~Shelhamer, H.~Shin, and J.~B. Tenenbaum, ``Infinite mixture
  prototypes for few-shot learning,'' \emph{arXiv preprint arXiv:1902.04552},
  2019.

\bibitem{kim2019edgelabelGNN}
J.~Kim, T.~Kim, S.~Kim, and C.~D. Yoo, ``Edge-labeling graph neural network for
  few-shot learning,'' pp. 11--20, 2019.

\bibitem{Wu_2019_ICCV}
Z.~Wu, Y.~Li, L.~Guo, and K.~Jia, ``{PARN}: Position-aware relation networks
  for few-shot learning,'' in \emph{IEEE International Conference on Computer
  Vision}, 2019, pp. 6659--6667.

\bibitem{zhang2019variationalfewshotlearn}
J.~Zhang, C.~Zhao, B.~Ni, M.~Xu, and X.~Yang, ``Variational few-shot
  learning,'' in \emph{IEEE International Conference on Computer Vision}, 2019,
  pp. 1685--1694.

\bibitem{mohri2018foundations}
M.~Mohri, A.~Rostamizadeh, and A.~Talwalkar, \emph{Foundations of machine
  learning}.\hskip 1em plus 0.5em minus 0.4em\relax MIT press, 2018.

\bibitem{maji2013fine}
S.~Maji, E.~Rahtu, J.~Kannala, M.~Blaschko, and A.~Vedaldi, ``Fine-grained
  visual classification of aircraft,'' \emph{arXiv preprint arXiv:1306.5151},
  2013.

\bibitem{krause20133d}
J.~Krause, M.~Stark, J.~Deng, and L.~Fei-Fei, ``{3D} object representations for
  fine-grained categorization,'' in \emph{IEEE International Conference on
  Computer Vision Workshops}, 2013, pp. 554--561.

\bibitem{khosla2011novel}
A.~Khosla, N.~Jayadevaprakash, B.~Yao, and F.-F. Li, ``Novel dataset for
  fine-grained image categorization: {Stanford} dogs,'' in \emph{CVPR Workshop
  on Fine-Grained Visual Categorization (FGVC)}, vol.~2, no.~1, 2011.

\bibitem{wah2011caltech}
C.~Wah, S.~Branson, P.~Welinder, P.~Perona, and S.~Belongie, ``The
  {Caltech-UCSD} {Birds}-200-2011 dataset,'' 2011.

\bibitem{selvaraju2017gradcam}
R.~R. Selvaraju, M.~Cogswell, A.~Das, R.~Vedantam, D.~Parikh, and D.~Batra,
  ``Grad-{CAM}: Visual explanations from deep networks via gradient-based
  localization,'' in \emph{IEEE International Conference on Computer Vision},
  2017, pp. 618--626.

\bibitem{mohri2012foundations}
M.~Mohri, A.~Rostamizadeh, and A.~Talwalkar, \emph{Foundations of machine
  learning}.\hskip 1em plus 0.5em minus 0.4em\relax MIT press, 2012.

\end{thebibliography}

\end{document}